\documentclass[pdflatex,sn-mathphys-num,onecolumn]{sn-jnl}

\usepackage{graphicx}
\usepackage{multirow}
\usepackage{amsmath,amssymb,amsfonts}
\usepackage{amsthm}
\usepackage{mathrsfs}
\usepackage[title]{appendix}
\usepackage{xcolor}
\usepackage{textcomp}
\usepackage{manyfoot}
\usepackage{booktabs}
\usepackage[ruled,vlined,linesnumbered]{algorithm2e}
\usepackage{listings}
\usepackage{tikz}
\usepackage{float}
\usepackage{caption}
\usetikzlibrary{arrows.meta,positioning,patterns,decorations.pathreplacing,calc}

\theoremstyle{thmstyleone}
\newtheorem{theorem}{Theorem}
\newtheorem{proposition}[theorem]{Proposition}
\theoremstyle{thmstyletwo}

\theoremstyle{thmstylethree}

\newtheorem{lemma}[theorem]{Lemma}

\begin{document}

\newgeometry{left=1in, right=1in, top=1in, bottom=1in}

\title[Flash-SemiCRF]{Streaming Structured Inference with Flash-SemiCRF}

\author*[1]{\fnm{Benjamin K.} \sur{Johnson}}\email{ben.johnson@vai.org}
\equalcont{These authors contributed equally to this work.}

\author[2]{\fnm{Thomas} \sur{Goralski}}
\equalcont{These authors contributed equally to this work.}

\author[1]{\fnm{Ayush} \sur{Semwal}}

\author[1]{\fnm{Hui} \sur{Shen}}
\makeatletter\g@addto@macro\artauthors{$^{\ddagger}$}\makeatother

\author*[2]{\fnm{H. Josh} \sur{Jang}}\email{josh.jang@vai.org}
\makeatletter\g@addto@macro\artauthors{$^{\ddagger}$}\makeatother

\makeatletter

\gdef\@equalconttext{{\centering $^{\dagger}$These authors contributed equally to this work.\\$^{\ddagger}$Co-senior authors\par}}
\makeatother

\affil[1]{\orgdiv{Department of Epigenetics}, \orgname{Van Andel Institute}, \orgaddress{\street{333 Bostwick Ave. NE}, \city{Grand Rapids}, \postcode{49503}, \state{MI}, \country{USA}}}

\affil[2]{\orgdiv{Department of Cell Biology}, \orgname{Van Andel Institute}, \orgaddress{\street{333 Bostwick Ave. NE}, \city{Grand Rapids}, \postcode{49503}, \state{MI}, \country{USA}}}

\abstract{
Semi-Markov Conditional Random Fields (semi-CRFs) assign labels to
segments of a sequence rather than to individual positions,
enabling exact inference over segment-level features and principled
uncertainty estimates at their boundaries.
However, existing implementations must materialize a large edge
potential tensor whose size grows with sequence length, maximum segment
length, and label count, becoming prohibitive for speech-scale state
spaces and intractable at genomic scales where sequences can exceed
100,000 positions. This memory bottleneck has limited the adoption of
exact segment-level inference for long sequences and large label sets.
We identify that the core inefficiency is materializing edge potentials
that can instead be evaluated on-the-fly from a compact prefix-sum
array, and make several improvements. First, replacing the stored edge
tensor with prefix-sum lookup reduces the memory footprint by a factor
proportional to the product of segment length and label count. Second,
a streaming forward-backward pass with checkpoint-boundary
normalization keeps working memory sublinear in sequence length while
preserving exact gradients. Third, zero-centered cumulative scores
control numerical drift and induce an adaptive duration prior under
label imbalance. We integrate these ideas into Flash-SemiCRF, a fused
Triton kernel that enables exact semi-CRF inference on previously
intractable problem sizes. Available at \url{https://github.com/biobenkj/flash-semicrf}.}


\maketitle

\newgeometry{left=1in, right=1in, top=1in, bottom=1in}

\clearpage

\section{Introduction}\label{sec:introduction}

Many sequence types are naturally segmented and prediction tasks need to go beyond per-position labels. Modern neural encoders learn rich cross-position dependencies, but their output layers remain per-position: the loss decomposes over individual positions, no constraint enforces that adjacent predictions form a valid segmentation, segment duration is never explicitly modeled, and there is no direct mechanism to obtain posterior marginals over segment boundaries or segment identities.

Semi-Markov conditional random fields (semi-CRFs)~\cite{sarawagi2004semicrf} formalizes discriminative segment-level prediction by scoring variable-length segments via content, duration, and transition potentials. This factorization provides three structural properties that per-position classifiers lack, such as: every position belongs to exactly one labeled segment, explicit duration modeling that captures label-dependent length distributions, and segment-level posterior marginals that quantify uncertainty over boundaries rather than individual positions. The integration of CRFs as neural network output layers~\cite{lample2016neural,ma2016endtoend, kong2016segmental,ye2018hybrid} demonstrated that structured inference and representation learning are complementary, with the CRF layer enforcing global coherence over the encoder's learned representations. Rush~\cite{rush2020pytorchstruct} provided a general-purpose semi-CRF
implementation (\texttt{pytorch-struct}), making the algorithm widely
accessible as a training and inference layer but remaining bound by the edge tensor at scale.

Semi-CRF inference requires materializing a dense edge tensor of size
$O(TKC^2)$ ($T$ positions, $K$ maximum segment duration, $C$ labels),
dominating both time and memory. For NLP tasks with $T$ in the tens to
hundreds and small $K$, this cost is manageable.
At genomic scales, modern CRF-based tools have primarily adopted
the per-position linear ($K{=}1$) case that avoids this edge tensor and the associated time and memory costs of a standard semi-CRF: Bonito~\cite{bonito} and
Dorado~\cite{dorado} decode raw nanopore signal, and
Tranquillyzer~\cite{semwal} annotates long-read RNA-seq, all at
sequence lengths routinely exceeding $10^5$. Earlier semi-CRF
applications to gene prediction -- CONRAD~\cite{decaprio2007conrad} and
CRAIG~\cite{bernal2007craig} -- relied on hand-crafted features and
CPU-bound Viterbi decode rather than serving as differentiable layers
under neural encoder training. Extending to the semi-Markov setting
multiplies the per-position cost by $K$; at genomic scales ($K \geq
10^3$), materializing the full edge tensor becomes infeasible on GPU
long before the sequential sweep over $T$ becomes the bottleneck.

Recent work from Zaratiana et al.~\cite{zaratiana2023fsemicrf} addresses this problem through pruning. Their Filtered Semi-CRF uses a local span classifier to discard irrelevant segments before running structured inference on the reduced graph. This exploits the sparsity of named entities in text, where most spans are non-entities. In contrast, for dense segmentation problems where every position belongs to a labeled segment such as those in genomics, the sparsity assumption does not typically hold and the full semi-CRF graph must be traversed.

FlashAttention~\cite{dao2022flashattention} demonstrated that for attention, the dominant cost is not arithmetic but memory access: by never materializing the $n \times n$ attention matrix and instead recomputing blocks on-chip, it achieves large wall-clock speedups despite higher FLOP counts. We apply the same IO-aware principle to semi-CRF inference. The standard implementation pre-computes and stores the full $O(TKC^2)$ edge tensor in HBM, then runs forward-backward over the stored result. At scale, this tensor dominates memory. We observe that when segment content scores decompose as sums over positions~\cite{sarawagi2004semicrf}, prefix sums reduce per-edge evaluation to $O(1)$, and the edge tensor need never be materialized. Each potential can instead be computed on-the-fly within a streaming dynamic programming (DP) kernel from an $O(TC)$ cumulative score array.

We present \textbf{Flash-SemiCRF}, a Triton-accelerated semi-CRF library built on this observation. The key ideas are:
\begin{enumerate}
\item \textbf{On-the-fly edge computation.} A prefix-sum decomposition of segment content scores enables $O(1)$-in-duration edge potential evaluation, reducing inference from memory-bound to compute-bound and dispensing with the $O(TKC^2)$ edge tensor entirely.

\item \textbf{Ring-buffer streaming.} The forward-backward algorithm operates through $K$-slot and $2K$-slot ring buffers for the forward and backward passes respectively, achieving $O(KC)$ working memory independent of sequence length $T$.

\item \textbf{Sublinear checkpointing.} Gradient checkpointing at $\sqrt{TK}$ intervals, following the recomputation paradigm of FlashAttention~\cite{dao2022flashattention} and the chunked scan of Mamba~\cite{gu2023mamba}, yields $O(\sqrt{T/K} \cdot KC)$ checkpoint memory that grows sublinearly in $T$.

\item \textbf{Globally-centered emissions.} A sequence-level baseline subtraction simultaneously stabilizes prefix sums at large $T$ and induces an adaptive, data-dependent duration prior that regularizes segmentation under label imbalance.
\end{enumerate}

We validate correctness through training convergence criteria, finite-difference gradient checks, and self-consistency between independent forward-backward implementations. We benchmark against \texttt{pytorch-struct}~\cite{rush2020pytorchstruct} and demonstrate that sparse or banded matrix representations are incompatible with semi-CRF inference due to the triangular, rather than diagonal, sparsity structure of the intermediate operators. We characterize the transition from memory-bound to compute-bound scaling as a function of $T$, $K$, and $C$. Flash-SemiCRF provides a structured decoding approach that can be applied to learned representations to produce segmentations that respect duration constraints and maintain segment-level coherence, properties that per-position classifiers typically discard.

\section{Terminological clarification}

A linear CRF (or linear-chain CRF) is a model in which every segment has duration one (K=1), so inference reduces to the standard per-position Viterbi/forward–backward recurrence. A linear scan is an algorithm that implements the semi-CRF dynamic program by sweeping left-to-right over positions, as opposed to tree-structured (parallel-scan) alternatives. The two are orthogonal: a linear scan can execute semi-CRF inference at both K=1 and $K\geq1$, and a linear CRF can in principle be computed by any backend. Throughout this paper, “linear CRF” refers to the K=1 model, while “linear scan” refers to the sequential sweep algorithm.


\section{Implementation}
\subsection{Edge Potential Decomposition}\label{sec:edge_decomposition}

The prefix-sum decomposition of segment content scores~\cite{sarawagi2004semicrf} enables $O(1)$-per-edge evaluation from an $O(TC)$ cumulative score array, reducing the memory footprint by a factor of $KC$ relative to the pre-materialized $O(TKC^2)$ edge tensor. Existing implementations use this decomposition to efficiently pre-compute the full edge tensor. We instead evaluate each potential on-the-fly within the streaming DP kernel, so the tensor is never materialized. The following subsections define the cumulative score array, the edge potential it supports, and boundary projections that extends the standard three-term decomposition.

\subsection{Cumulative Centered Scores}\label{sec:cumulative_scores}

Let $f_\theta(t, c)$ denote the raw projected encoder emissions at token $t$ for label $c$. Before constructing prefix sums, we subtract a \emph{sequence-level emission baseline}:
\begin{equation}
\label{eq:baseline}
\nu_{b,c} = \frac{1}{L_b} \sum_{u=0}^{L_b-1} f_\theta(u, c), \qquad
\bar{f}_\theta(t, c) = f_\theta(t, c) - \nu_{b,c}
\end{equation}
where $L_b$ is the true (unpadded) sequence length for batch element $b$.\newline

\noindent The cumulative scores are then defined over the centered emissions:
\begin{equation}\label{eq:cumscores}
\mathcal{S}_{0, c} = 0, \qquad \mathcal{S}_{t, c} = \sum_{u=0}^{t-1} \bar{f}_\theta(u, c)
\end{equation}

\noindent The content score for a segment $[s, s+k)$ labeled $c$ decomposes as a constant-time lookup:
\begin{equation}\label{eq:content_centered}
\mathcal{S}_{s+k, c} - \mathcal{S}_{s, c}
= \sum_{u=s}^{s+k-1} f_\theta(u, c) \;-\; \nu_{b,c} \cdot k
\end{equation}

\noindent The data-dependent term $-\nu_{b,c} \cdot k$ introduced by centering couples label identity to segment duration, inducing an adaptive duration prior described in Section~\ref{sec:centering}.

\emph{Numerical motivation.} Without centering, cumulative scores grow as $\mathcal{S}_{T,c} = O(T)$ when encoder emissions have nonzero per-label means. At genomic scale ($T > 10^5$), segment scores $\mathcal{S}_{t+k,c} - \mathcal{S}_{t,c}$ risk catastrophic cancellation, when both operands reach magnitude $O(T)$ while their difference is $O(1)$. Subtracting the per-label mean before accumulation produces a zero-mean random walk whose cumulative sums have standard deviation $O(\sqrt{T})$, reducing the dynamic range by a factor of $\sqrt{T}$.

\subsection{Edge Potentials}\label{sec:edge_potentials}

The edge potential for a segment ending at position $t$ with duration $k$, transitioning from source label $c'$ to destination label $c$, is:
\begin{equation}\label{eq:streaming_potential}
\tilde{\psi}(t, k, c, c') =
\underbrace{\bigl(\mathcal{S}_{t, c} - \mathcal{S}_{t-k, c}\bigr)}_{\text{centered content}} +
\underbrace{\mathcal{B}_{k, c}}_{\text{duration bias}} +
\underbrace{\mathcal{T}_{c', c}}_{\text{transition}}
+ \underbrace{\mathcal{P}^{\text{start}}_{t-k, c} + \mathcal{P}^{\text{end}}_{t-1, c}}_{\text{boundary}}
\end{equation}

\noindent where $k \in \{1, \ldots, K\}$ is the segment duration, $c$ is the destination label, and $c'$ is the source label. All terms in Equation~\eqref{eq:streaming_potential} are either precomputed arrays ($\mathcal{S}$, $\mathcal{B}$, $\mathcal{T}$, $\mathcal{P}$) or scalar indices, so each edge potential is evaluated in $O(1)$ time regardless of duration $k$. The $O(TC)$ cumulative score array replaces the $O(TKC^2)$ pre-materialized edge tensor, reducing the memory footprint by a factor of $KC$.

\subsection{Boundary Projections}\label{sec:boundary_projections}

The content score $\mathcal{S}_{t,c} - \mathcal{S}_{t-k,c}$ aggregates encoder emissions uniformly over all positions in a segment, but carries no information about the \emph{positions} at which segments begin or end. The transition matrix $\mathcal{T}_{c',c}$ captures label-to-label preferences but is position-independent. For applications where segment boundaries have distinctive local signatures, such as exon-intron transitions in genomic annotation, beat onsets in physiological signals, etc., explicit boundary features can improve segmentation accuracy.

The optional boundary projections score segment start and end positions directly from the encoder's hidden representation via independent linear heads, $\mathcal{P}^{\text{start}}_{t, c} = W^{\text{start}}_c \cdot \mathbf{h}_t$ and analogously for $\mathcal{P}^{\text{end}}$, with $\mathbf{h}_t \in \mathbb{R}^d$ the encoder hidden state at position $t$ and $W^{\text{start}}, W^{\text{end}} \in \mathbb{R}^{C \times d}$ learned projection matrices. A segment $[s, s+k)$ with label $c$ receives additive boundary contributions $\mathcal{P}^{\text{start}}_{s,c} + \mathcal{P}^{\text{end}}_{s+k-1,c}$.

Separate start and end projections reflect a structural property of many segmentation problems: segment entry and exit involve categorically different signals (e.g., the start of a "person" identity in language has different best indicators, the first name, than the end indicators, the last name). The separable form also has a computational advantage: each endpoint requires only a single lookup into a precomputed $(B, T, C)$ array, whereas a joint parameterization $\phi(c, \mathbf{h}_s, \mathbf{h}_{s+k-1})$ would require $O(TKC)$ precomputation over all valid endpoint pairs. The projections operate on the full encoder hidden states $\mathbf{h}_t$ rather than the label-space scores $f_\theta(t,c)$, allowing boundary features to capture context that may not be linearly decodable in label space. When disabled ($\mathcal{P}^{\text{start}} = \mathcal{P}^{\text{end}} = 0$), Equation~\eqref{eq:streaming_potential} reduces to the standard three-term decomposition.

\paragraph{Relationship to classical semi-CRF features.}
In the classical formulation~\cite{sarawagi2004semicrf}, segment features have the general form $g'_k(y_j, y_{j-1}, \mathbf{x}, t_j, u_j)$, where each feature may depend jointly on the segment label, the preceding label, the input, and both endpoint positions, at the cost of materializing the full $O(TKC^2)$ edge tensor. Our boundary projections are the streaming-compatible factorization of this capability: boundary scores depend on position $t$ and destination label $c$, but \emph{not} on segment duration $k$ or source label $c'$, restricting boundary features to the separable form $\phi(y_j, \mathbf{x}, t_j) + \phi'(y_j, \mathbf{x}, u_j)$. This is the price of the $O(KC)$ memory guarantee, $\mathcal{P}^{\text{start}}$ and $\mathcal{P}^{\text{end}}$ are precomputed before the DP loop and cannot condition on quantities only known at recurrence time.

\section{Emission Baseline Centering}\label{sec:centering}

The emission baseline subtraction (Equation~\ref{eq:baseline}) serves a dual purpose. Numerically, it controls cumulative sum magnitude as described in Section~\ref{sec:cumulative_scores}. But per-label centering is not path-invariant, and it simultaneously induces an adaptive duration prior with useful regularization properties. We analyze this modeling effect here so that the reader enters the algorithm with a complete understanding of what the edge potentials encode.

\subsection{The Adaptive Duration Prior}\label{sec:adaptive_prior}

A path-invariant centering, such as subtracting a per-position scalar $b_t = \max_c f_\theta(t,c)$ shared across labels, would cancel in the partition function because every valid segmentation covers each position exactly once. Per-label centering with $\nu_{b,c}$ does not cancel, because different segmentations assign different labels to different positions and $\nu_{b,c}$ varies across labels.

The effective segment score under centering (Equation~\ref{eq:content_centered}) is:
\begin{equation}\label{eq:effective_score}
S_{\text{eff}}(c, s, k) = \sum_{u=s}^{s+k-1} f_\theta(u, c) + \mathcal{B}_{k,c} - \nu_{b,c} \cdot k
\end{equation}

\noindent The term $-\nu_{b,c} \cdot k$ acts as a \emph{data-dependent duration prior}: labels with high average emission $\nu_{b,c}$ incur a larger penalty for long segments, while labels with low average emission receive a relative boost. This is equivalent to a canonical semi-CRF with an effective duration model $\mathcal{B}^{\text{eff}}_{k,c} = \mathcal{B}_{k,c} - \nu_{b,c} \cdot k$ that adapts per-sequence based on the global emission statistics.

This decomposition has an empirical Bayes-like feature~\cite{efron2010largescale}: $\mathcal{B}_{k,c}$ captures population-level duration priors across training sequences, while $\nu_{b,c}$ provides per-sequence adjustment. Since $\mathcal{B}_{k,c}$ can absorb the average centering effect across training, the adaptive prior primarily regularizes per-sequence variation under label imbalance.

\subsection{Worked Example and Ablation}

Table~\ref{tab:toy_centering} illustrates the asymmetric regularization under label imbalance using a synthetic sequence with three genomic labels. The duration penalty $-\nu_{b,c} \cdot k$ grows linearly in $k$, so it primarily affects long segments of the dominant label, and are often the degenerate segmentations one wishes to suppress. Rare labels with low or negative baselines incur minimal penalty or a slight boost, and their centered emissions at active positions have higher contrast ($+8.12$ for \textsc{Promoter}) than for dominant labels ($+0.75$ for \textsc{Intron}), because centering subtracts a baseline close to the active-position emission for prevalent labels.

\begin{table}[h]
\centering
\small
\begin{tabular}{l c c c c c}
\hline
Label $c$ & $f_\theta$ (active) & $f_\theta$ (inactive) & Prevalence & $\nu_{b,c}$ & $-\nu_{b,c} \cdot 100$ \\
\hline
\textsc{Intron}   & $+4.0$ & $-1.0$ & 85\% & $+3.25$ & $-325$ \\
\textsc{Exon}     & $+5.0$ & $-0.5$ & 14\% & $+0.27$ & $-27$ \\
\textsc{Promoter} & $+8.0$ & $-0.2$ & 1\%  & $-0.12$ & $+12$ \\
\hline
\end{tabular}
\caption{Asymmetric regularization under label imbalance ($T = 10{,}000$, $C = 3$).
The baseline $\nu_{b,c}$ is the weighted average of active and inactive emissions.
The dominant label (\textsc{Intron}) receives a strong duration penalty ($-325$ at $k{=}100$), while the rare label (\textsc{Promoter}) receives a slight duration \emph{bonus} ($+12$) because its baseline is negative. Centered emissions at active positions ($f_\theta^{(\text{active})} - \nu_{b,c}$) are $+0.75$ (\textsc{Intron}), $+4.73$ (\textsc{Exon}), $+8.12$ (\textsc{Promoter})--centering amplifies signal-to-background contrast for rare labels.}
\label{tab:toy_centering}
\end{table}

In ablation experiments (Table~\ref{tab:implicit-prior}) on synthetic label-imbalanced sequences ($C = 3$, $T = 2{,}000$, proportions 75/15/10\%), Viterbi decoding under mean centering recovered 59 segments of the rarest label versus 19 under path-invariant centering, and 64 versus 35 for the second-rarest, while the dominant label's segment count decreased from 142 to 90. Under balanced proportions (0.33 each), segment counts converge across centering modes.

\begin{table}[h]
\centering
\small
\begin{tabular}{l r r r r r r r r}
\toprule
 & & \multicolumn{3}{c}{Duration penalty $-\nu \cdot k$}
 & \multicolumn{2}{c}{Imbalanced} & \multicolumn{2}{c}{Balanced} \\
\cmidrule(lr){3-5} \cmidrule(lr){6-7} \cmidrule(lr){8-9}
Label & $\nu_{b,c}$ & $k{=}10$ & $k{=}25$ & $k{=}50$
 & \textsc{mean} & \textsc{none}
 & \textsc{mean} & \textsc{none} \\
\midrule
  INTRON & 2.021 & -20.21 & -50.53 & -101.05 & 90 & 142 & 80 & 84 \\
  EXON & -0.413 & 4.13 & 10.32 & 20.64 & 64 & 35 & 80 & 80 \\
  UTR & -0.603 & 6.03 & 15.07 & 30.14 & 59 & 19 & 76 & 77 \\
\bottomrule
\end{tabular}
\caption{Implicit duration prior from \textsc{mean} centering under label imbalance.
  The per-label temporal mean $\nu_{b,c}$ creates a duration penalty $-\nu_{b,c} \cdot k$
  that penalizes long segments of prevalent labels.
  Right columns show decoded segment counts: under imbalance (75/15/10\%),
  mean centering redistributes segments from the dominant to rare labels;
  under balanced proportions (0.33 each), the effect vanishes.
  Separately, cumulative score magnitudes at $T{=}50{,}000$ confirm that
  mean centering grows as $O(\sqrt{T})$ (peak$/\sqrt{T} \approx 1.28$)
  versus $O(T)$ without centering (peak$/\sqrt{T} \approx 484$).}
\label{tab:implicit-prior}
\end{table}

\emph{Architectural context.} When the semi-CRF serves as an output decoder, the encoder can learn representations calibrated to the centering, and the adaptive prior provides robustness to emission scale miscalibration during early training. When semi-CRF marginals serve as input features for a downstream model, a path-invariant centering ($b_t = \max_c f_\theta(t,c)$) is preferable to preserve the canonical semi-CRF distribution.

\section{Streaming Forward-Backward Algorithm}\label{sec:streaming_algorithm}

We now present the streaming forward-backward algorithm that maintains only $O(KC)$ working memory through ring buffers, with gradient checkpointing at $\Delta \approx \sqrt{TK}$ intervals. The design follows the tiling and recomputation pattern of FlashAttention~\cite{dao2022flashattention} and the chunked parallel scan of Mamba~\cite{gu2023mamba}, adapted to the semi-Markov recurrence structure.

\subsection{Forward Scan with Ring Buffer}\label{sec:forward_scan}

Algorithm~\ref{alg:streaming_forward} maintains a ring buffer $\boldsymbol{\alpha} \in \mathbb{R}^{K \times C}$ storing the $K$ most recent log-forward messages (Figure~\ref{fig:memory_layout}). At each position $t$, the recurrence looks back at most $K$ positions, so a ring buffer indexed by $t \bmod K$ suffices--the oldest entry is always the one about to be overwritten.

At checkpoint boundaries (every $\Delta$ positions), the algorithm normalizes the ring buffer by subtracting the current maximum, accumulating the shift into $\mathcal{N}_{\text{accum}}$. This prevents the log-forward messages from drifting to extreme magnitudes over long sequences~\cite{dao2022flashattention}. The full ring buffer state $\Omega_i$ and cumulative normalizer $\mathcal{N}_i$ are saved at each checkpoint for use in the backward pass. For variable-length batches, the normalization shift is masked to zero for positions beyond the true sequence length $L_b$, preventing phantom normalization of padding positions.

\begin{algorithm}[h]
\caption{Streaming Semi-CRF Forward Scan}\label{alg:streaming_forward}
\KwIn{$\mathcal{S}$, $\mathcal{T}$, $\mathcal{B}$, $\mathcal{P}^{\text{start}}$, $\mathcal{P}^{\text{end}}$ (optional), checkpoint interval $\Delta$}
\KwOut{$\log Z$, checkpoints $(\Omega, \mathcal{N})$}
\BlankLine
$\boldsymbol{\alpha} \gets -\infty$; \quad $\boldsymbol{\alpha}[0, :] \gets 0$; \quad $\mathcal{N}_{\text{accum}} \gets 0$\;
$\Omega[0] \gets \boldsymbol{\alpha}$; \quad $\mathcal{N}[0] \gets 0$\;
\BlankLine
\For{$t \gets 1$ \KwTo $T$}{
    $\mathbf{v}_t \gets -\infty \in \mathbb{R}^C$\;
    \For{$k \gets 1$ \KwTo $\min(K, t)$}{
        $\tilde{\boldsymbol{\alpha}}_{\text{prev}} \gets \boldsymbol{\alpha}[(t-k) \bmod K, :]$\;
        $\mathbf{h} \gets \bigl(\mathcal{S}_{t,:} - \mathcal{S}_{t-k,:}\bigr) + \mathcal{B}_{k,:} + \mathcal{P}^{\text{start}}_{t-k,:} + \mathcal{P}^{\text{end}}_{t-1,:}$\;
        $\mathbf{E} \gets \mathbf{h}[:, \text{None}] + \mathcal{T}^\top$\;
        $\mathbf{s}_k \gets \text{LSE}(\tilde{\boldsymbol{\alpha}}_{\text{prev}}[\text{None}, :] + \mathbf{E}, \text{axis}=1)$\;
        $\mathbf{v}_t \gets \text{LSE}(\mathbf{v}_t, \mathbf{s}_k)$\;
    }
    $\boldsymbol{\alpha}[t \bmod K, :] \gets \mathbf{v}_t$\;
    \BlankLine
    \If(\tcp*[f]{Normalize at checkpoint~\cite{dao2022flashattention}}){$t \bmod \Delta = 0$}{
        $n \gets t / \Delta$\;
        $s_b \gets \begin{cases} \max_c \mathbf{v}_{t,b}(c) & \text{if } t \le L_b \\ 0 & \text{otherwise} \end{cases}$\;
        $\mathcal{N}_{\text{accum}} \gets \mathcal{N}_{\text{accum}} + s_b$; \quad $\boldsymbol{\alpha} \gets \boldsymbol{\alpha} - s_b$\;
        $\Omega[n] \gets \boldsymbol{\alpha}$; \quad $\mathcal{N}[n] \gets \mathcal{N}_{\text{accum}}$\;
    }
}
\BlankLine
$\log Z \gets \text{LSE}(\boldsymbol{\alpha}[T \bmod K, :]) + \mathcal{N}_{\text{accum}}$\;
\Return{$\log Z$, $(\Omega, \mathcal{N})$}\;
\end{algorithm}

\subsection{Backward Scan and the $2K$ Ring Buffer}\label{sec:backward_scan}

The backward pass processes checkpoint segments in reverse order. For each segment $[i\Delta, (i{+}1)\Delta)$, the algorithm restores $\boldsymbol{\alpha}$ from the saved checkpoint $\Omega_i$, recomputes the local forward messages within the segment, then scans $\boldsymbol{\beta}$ right-to-left, computing marginals and accumulating gradients at each position (Figure~\ref{fig:memory_layout}b). The saved normalizer $\mathcal{N}_i$ restores the true scale of $\boldsymbol{\alpha}$ when computing marginals via $\log \mu = \tilde{\alpha} + \tilde{\psi} + \tilde{\beta} + \mathcal{N}_i - \log Z$. The full backward algorithm is given in Appendix~\ref{sec:appendix_streaming}.

While the forward pass uses a ring buffer of size $K$ (sufficient since we only look back $K$ positions), the backward pass requires a ring buffer of size $2K$. This asymmetry arises from a write-before-read hazard in the backward access pattern: at position $t$ we \emph{read} $\beta$ values from positions $\{t{+}1, \ldots, t{+}K\}$ and \emph{write} $\beta[t]$, so with $K$ slots the write at $t$ collides with the read at $t{+}K$ (both map to $t \bmod K$), overwriting $\beta[t{+}K]$ before it is read. Doubling the buffer to $2K$ slots separates write and read indices for the entire lookback window, eliminating the conflict. The additional $KC$ float memory overhead is negligible compared to the $O(\sqrt{T/K} \cdot KC)$ checkpoint storage.

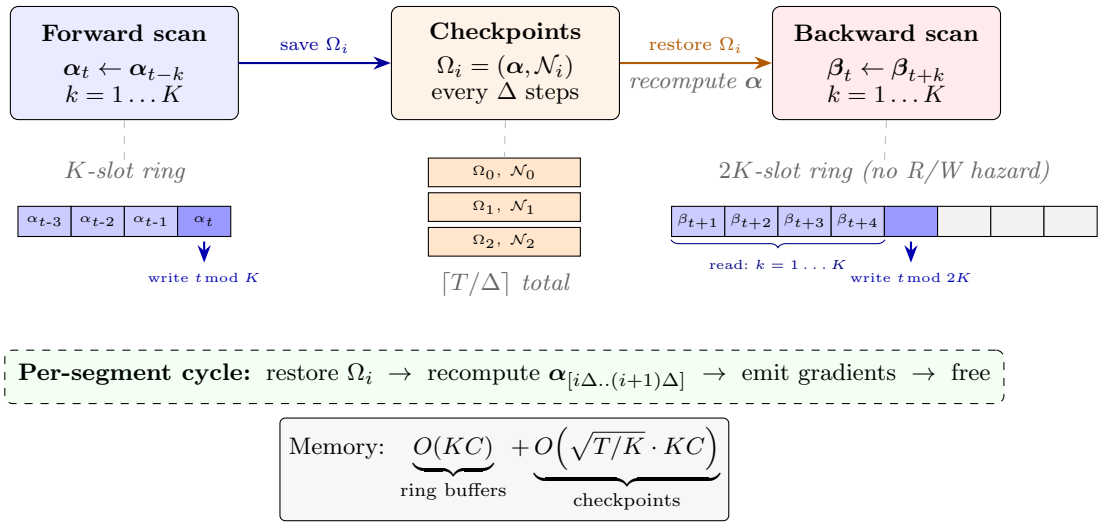
\begin{figure}[h]
\centering
\begin{tikzpicture}[
    >=Stealth,
    block/.style={draw, rounded corners=3pt, minimum width=3.0cm, minimum height=1.5cm,
                  align=center, font=\small},
    arrow/.style={->, thick, >=Stealth},
    dataarrow/.style={arrow, blue!60!black},
    ckptarrow/.style={arrow, orange!70!black},
    lbl/.style={font=\footnotesize, align=center},
    annot/.style={font=\scriptsize\itshape, text=black!60},
    slot/.style={draw, minimum width=0.7cm, minimum height=0.4cm, inner sep=0pt, font=\tiny},
    slotfill/.style={slot, fill=blue!20},
    slotwrite/.style={slot, fill=blue!40},
    slotckpt/.style={slot, fill=orange!20},
    slotempty/.style={slot, fill=black!5},
    connector/.style={draw, thin, black!25, dashed},
]
 
\node[block, fill=blue!8] (fwd) at (0, 0) {
    \textbf{Forward scan}\\[2pt]
    $\boldsymbol{\alpha}_{t} \leftarrow \boldsymbol{\alpha}_{t-k}$\\[-1pt]
    {\scriptsize $k = 1 \ldots K$}
};
 
\node[block, fill=orange!10, right=2.0cm of fwd] (ckpt) {
    \textbf{Checkpoints}\\[2pt]
    $\Omega_i = (\boldsymbol{\alpha}, \mathcal{N}_i)$\\[-1pt]
    {\scriptsize every $\Delta$ steps}
};
 
\node[block, fill=red!8, right=2.0cm of ckpt] (bwd) {
    \textbf{Backward scan}\\[2pt]
    $\boldsymbol{\beta}_{t} \leftarrow \boldsymbol{\beta}_{t+k}$\\[-1pt]
    {\scriptsize $k = 1 \ldots K$}
};
 
\draw[dataarrow] (fwd.east) -- (ckpt.west)
    node[midway, above, lbl] {save $\Omega_i$};
\draw[ckptarrow] (ckpt.east) -- (bwd.west)
    node[midway, above, lbl] {restore $\Omega_i$}
    node[midway, below, annot] {recompute $\boldsymbol{\alpha}$};
 
 
\def\insetlabely{-1.45}
\def\insetsloty{-2.1}
 
\node[annot] at (0, \insetlabely) {$K$-slot ring};
\foreach \i/\sty in {0/slotfill, 1/slotfill, 2/slotfill, 3/slotwrite} {
    \pgfmathsetmacro{\xoff}{-1.05 + \i*0.7}
    \node[\sty] (fring\i) at (\xoff, \insetsloty) {};
}
\node[font=\tiny] at (fring0) {$\alpha_{t\text{-}3}$};
\node[font=\tiny] at (fring1) {$\alpha_{t\text{-}2}$};
\node[font=\tiny] at (fring2) {$\alpha_{t\text{-}1}$};
\node[font=\tiny] at (fring3) {$\alpha_{t}$};
\draw[->, blue!70!black, thick] (fring3.south) ++(0,-0.06) -- ++(0,-0.28)
    node[below, font=\tiny, blue!70!black] {write $t\!\bmod K$};
\draw[connector] (fwd.south) -- (0, {\insetlabely + 0.15});
 
\pgfmathsetmacro{\ckptx}{5.0}
\foreach \i/\yoff in {0/0, 1/-0.46, 2/-0.92} {
    \pgfmathsetmacro{\ypos}{\insetlabely + \yoff}
    \node[slotckpt, minimum width=2.0cm, minimum height=0.38cm] (cslot\i) at (\ckptx, \ypos) {};
}
\node[font=\tiny] at (cslot0) {$\Omega_0,\; \mathcal{N}_0$};
\node[font=\tiny] at (cslot1) {$\Omega_1,\; \mathcal{N}_1$};
\node[font=\tiny] at (cslot2) {$\Omega_2,\; \mathcal{N}_2$};
\node[annot, anchor=north] at ([yshift=-0.08cm]cslot2.south) {$\lceil T/\Delta \rceil$ total};
\draw[connector] (ckpt.south) -- (cslot0.north);
 
\pgfmathsetmacro{\bwdx}{10.0}
\node[annot] at (\bwdx, \insetlabely) {$2K$-slot ring (no R/W hazard)};
\foreach \i/\sty in {0/slotfill, 1/slotfill, 2/slotfill, 3/slotfill, 4/slotwrite, 5/slotempty, 6/slotempty, 7/slotempty} {
    \pgfmathsetmacro{\xoff}{\bwdx - 2.45 + \i*0.7}
    \node[\sty] (bring\i) at (\xoff, \insetsloty) {};
}
\node[font=\tiny] at (bring0) {$\beta_{t\text{+}1}$};
\node[font=\tiny] at (bring1) {$\beta_{t\text{+}2}$};
\node[font=\tiny] at (bring2) {$\beta_{t\text{+}3}$};
\node[font=\tiny] at (bring3) {$\beta_{t\text{+}4}$};
\draw[decorate, decoration={brace, amplitude=3pt, raise=1pt, mirror}, blue!50!black]
    (bring0.south west) -- (bring3.south east)
    node[midway, below=5pt, font=\tiny, blue!50!black] {read: $k = 1 \ldots K$};
\draw[->, blue!70!black, thick] (bring4.south) ++(0,-0.06) -- ++(0,-0.28)
    node[below, font=\tiny, blue!70!black] {write $t\!\bmod 2K$};
\draw[connector] (bwd.south) -- (\bwdx, {\insetlabely + 0.15});
 
\node[draw, dashed, rounded corners=3pt, fill=green!5,
      align=center, font=\scriptsize, inner sep=5pt,
      anchor=north]
    (cycle) at (5.0, -3.8)
    {\textbf{Per-segment cycle:}\;
     restore $\Omega_i \;\to\;$ recompute $\boldsymbol{\alpha}_{[i\Delta..(i{+}1)\Delta]}$
     $\;\to\;$ emit gradients $\;\to\;$ free};
 
\node[draw, rounded corners=2pt, fill=black!3, font=\scriptsize, align=left,
      anchor=north]
    (mem) at (5.0, -4.7)
    {Memory:\; $\underbrace{O(KC)}_{\text{ring buffers}}
     + \underbrace{O\!\left(\sqrt{T/K}\cdot KC\right)}_{\text{checkpoints}}$};
 
\end{tikzpicture}
\caption{Memory layout for streaming semi-CRF inference.
The forward scan maintains a $K$-slot ring buffer indexed by $t \bmod K$,
checkpointing the full buffer state $\Omega_i$ and cumulative log-normalizer
$\mathcal{N}_i$ every $\Delta$ positions.
During the backward pass, each segment is processed independently:
$\boldsymbol{\alpha}$ is restored from $\Omega_i$ and recomputed locally,
while $\boldsymbol{\beta}$ is scanned right-to-left through a $2K$-slot
ring buffer that prevents the write-before-read hazard
(Section~\ref{sec:backward_scan}).
Total memory is $O(KC)$ for ring buffers plus
$O(\sqrt{T/K} \cdot KC)$ for checkpoints, sublinear in $T$.}
\label{fig:memory_layout}
\end{figure}

\subsection{Complexity}\label{sec:complexity}

Table~\ref{tab:complexity} summarizes the time and memory complexity of the streaming algorithm. The dominant working memory is the $O(KC)$ ring buffer, which is independent of sequence length $T$. Checkpoints contribute $O(\sqrt{T/K} \cdot KC)$, which grows sublinearly in $T$. The total time is $O(TKC^2)$, the same as a standard semi-CRF
forward-backward pass. The dominant memory reduction is replacing
the $O(TKC^2)$ edge tensor with the $O(TC)$ prefix-sum array (a
factor of $KC$). The ring buffer further reduces forward-message
storage from $O(TC)$ to $O(KC)$, independent of sequence length.
At $T = 10^6$, $K = 200$, $C = 6$, the edge tensor would require
${\sim}29$\,GB in float32; the prefix-sum array requires
${\sim}24$\,MB, and the ring buffer ${\sim}10$\,KB.

\begin{table}[h]
\centering
\begin{tabular}{l c c}
\hline
\textbf{Component} & \textbf{Time} & \textbf{Memory} \\
\hline
Prefix-sum array & $O(TC)$ & $O(TC)$ \\
Forward scan & $O(TKC^2)$ & $O(KC)$ ring buffer \\
Checkpoints & -- & $O(\sqrt{T/K} \cdot KC)$ \\
Backward scan & $O(TKC^2)$ & $O(KC)$ ring buffer ($2K$ slots) \\
Alpha recompute & $O(TKC^2)$ & $O(\Delta \cdot C)$ per segment \\
Gradient workspace & -- & $O(B \cdot N_{\text{ckpt}} \cdot KC^2)$ \\
\hline
\textbf{Total} & $O(TKC^2)$ &
    $O(KC + \sqrt{T/K} \cdot KC)$\textsuperscript{$\dagger$} \\
\hline
\end{tabular}
\caption{Complexity analysis. The DP working memory (ring buffers) is
$O(KC)$, independent of $T$; checkpoint memory grows as $O(\sqrt{T})$.
The $O(TC)$ prefix-sum array is the dominant memory term but replaces
the $O(TKC^2)$ edge tensor, a reduction by factor $KC$.
$N_{\text{ckpt}} = \lceil T/\Delta \rceil$, $B$ is batch size.
$\dagger$Working memory only; excludes the $O(TC)$ prefix-sum array
(precomputed, not part of the streaming DP) and the $O(B \cdot N \cdot KC^2)$
gradient workspace.}
\label{tab:complexity}
\end{table}

\subsection{Adaptive Loop Tiling}\label{sec:loop_tiling}
The marginal computation requires a $(C_{\text{PAD}} \times
C_{\text{PAD}})$ matrix per $(t, k)$ pair, which at $C_{\text{PAD}} =
64$ demands approximately 384 registers per thread. On our L40S GPUs
(65,536 registers per SM), this exceeds the available register file at
$>$4 warps and causes spilling to slow local memory. We tile the destination label dimension in blocks of size $\tau$, reducing peak register demand to ${\sim}120$ per thread and enabling 4--8 warps without spilling. The $\beta$ reduction across tiles uses the online logsumexp pattern from FlashAttention~\cite{dao2022flashattention}. The tile size $\tau$ is selected adaptively based on $C$ to balance compile time, register pressure, and iteration count (Table~\ref{tab:tile_sizing}).

\begin{table}[h]
\centering
\small
\begin{tabular}{c c c l}
\hline
$C_{\text{PAD}}$ & $\tau$ & Iterations & Rationale \\
\hline
$\leq 8$ & 4 & 2 & Minimal iteration count \\
$\leq 16$ & 8 & 2 & Minimal iteration count \\
$32$ & 16 & 2 & Balanced \\
$64$ & 16 & 4 & Moderate register pressure \\
$\geq 128$ & 32 & $\leq 8$ & Bounded compile time \\
\hline
\end{tabular}
\caption{Adaptive tile size selection. The algorithm bounds iteration count to $\leq 8$ even at $C = 256$ while keeping register pressure manageable.}
\label{tab:tile_sizing}
\end{table}

\subsection{Correctness validation.}\label{sec:correctness_validation}
Because the Triton kernel uses a different floating-point reduction order than the PyTorch reference, bitwise agreement is not expected. Instead, we validate through three strategies (Table~\ref{tab:correctness-supp}). \emph{Self-consistency:} boundary and emission marginals from the backward pass must satisfy the following probabilistic invariants, 1) marginals bounded as $P(c \mid t) \in [0,1]$, 2) class marginals normalized as $\sum_c P(c \mid t) = 1$ at each position, 3) total mass conserved as $\sum_{t,c} P(c \mid t) = T$, and 4) zero mass at padded positions. All invariants hold at genome scale ($B{=}142$, $T{=}100{,}000$, $C{=}24$, $K{=}100$) with per-position deviation of $1.6 \times 10^{-4}$. \emph{Finite-difference gradient checking:} for each element of $\mathcal{S}$ (cumulative scores), $\mathcal{T}$ (transition matrix), and $\mathcal{B}$ (duration bias), we perturb by $\pm\varepsilon$
($\varepsilon = 10^{-3}$) and compute the central-difference approximation
$(f(x{+}\varepsilon) - f(x{-}\varepsilon)) / 2\varepsilon$ using only
forward-pass evaluations. This numerical gradient, independent of the backward
kernel, serves as ground truth against which the Triton autograd output is
compared. All parameters achieve cosine similarity $\geq 0.9999$ and
normalized maximum error $< 5 \times 10^{-5}$, confirming that the custom
backward kernels produce correct gradients. \emph{Training convergence comparison:} models initialized with the same random seeds were trained with the streaming PyTorch, streaming Triton, and \texttt{pytorch-struct} linear scan backends reaching the same final negative log-likelihood (relative difference $< 10^{-9}$) with loss-curve cosine similarity of $1.0$ after 100 epochs, while the Triton backend runs approximately $5.7\times$ faster.

\begin{table}[h]
\centering
\small
\begin{tabular}{l r r r r}
\toprule
\multicolumn{5}{l}{\textbf{(a) Self-consistency -- large scale ($B{=}4$, $T{=}2,000$, $C{=}32$, $K{=}50$)}} \\
\midrule
\multicolumn{2}{l}{Invariant} & \multicolumn{3}{l}{Observed} \\
\midrule
  \multicolumn{2}{l}{Boundary marginals $\in [0, 1]$} & \multicolumn{3}{l}{$[0.9699,\; 1.0000]$} \\
  \multicolumn{2}{l}{Expected segments $\sum_t p(\mathrm{bdy} \mid t) \in [1, T]$} & \multicolumn{3}{l}{$\bar{x} = 1940.6$, range $[1940.6,\; 1940.6]$} \\
  \multicolumn{2}{l}{$\sum_c P(c \mid t) = 1$} & \multicolumn{3}{l}{max dev.\ $= 1.5\times 10^{-6}$} \\
  \multicolumn{2}{l}{$\sum_{t,c} P(c \mid t) \approx T$} & \multicolumn{3}{l}{$2{,}000 \pm 1.5\times 10^{-3}$ (0.0001\%)} \\
  \multicolumn{2}{l}{Emission marginals $\geq 0$} & \multicolumn{3}{l}{observed $\min = 0.0179$} \\
  \multicolumn{2}{l}{Padding-mask consistency} & \multicolumn{3}{l}{all padded positions $= 0$} \\
\midrule
\multicolumn{5}{l}{\textbf{(b) Self-consistency -- genome scale ($B{=}142$, $T{=}100,000$, $C{=}24$, $K{=}100$)}} \\
\midrule
\multicolumn{2}{l}{Invariant} & \multicolumn{3}{l}{Observed} \\
\midrule
  \multicolumn{2}{l}{Boundary marginals $\in [0, 1]$} & \multicolumn{3}{l}{$[0.9597,\; 1.0000]$} \\
  \multicolumn{2}{l}{Expected segments $\sum_t p(\mathrm{bdy} \mid t) \in [1, T]$} & \multicolumn{3}{l}{$\bar{x} = 96085.4$, range $[96085.3,\; 96085.5]$} \\
  \multicolumn{2}{l}{$\sum_c P(c \mid t) = 1$} & \multicolumn{3}{l}{max dev.\ $= 1.6\times 10^{-4}$} \\
  \multicolumn{2}{l}{$\sum_{t,c} P(c \mid t) \approx T$} & \multicolumn{3}{l}{$100{,}000 \pm 8.10$ (0.0081\%)} \\
  \multicolumn{2}{l}{Emission marginals $\geq 0$} & \multicolumn{3}{l}{observed $\min = 0.0208$} \\
  \multicolumn{2}{l}{Padding-mask consistency} & \multicolumn{3}{l}{all padded positions $= 0$} \\
\midrule
\multicolumn{5}{l}{\textbf{(c) Finite-difference gradient validation}} \\
\midrule
Parameter & Elements & Cosine sim. & Norm.\ max err & Max $|\Delta|$ \\
\midrule
  cum\_scores & 1,600 & 1.000000 & $2.1\times 10^{-5}$ & $1.5\times 10^{-6}$ \\
  transition & 256 & 1.000000 & $4.6\times 10^{-6}$ & $2.3\times 10^{-6}$ \\
  duration\_bias & 400 & 1.000000 & $2.6\times 10^{-6}$ & $1.7\times 10^{-5}$ \\
\midrule
\multicolumn{5}{l}{\textbf{(d) Training convergence}} \\
\midrule
Backend & Final NLL & Rel.\ diff & Curve cos. & Time (s) \\
\midrule
  \textsc{pytorch} & 1606.71 & -- & -- & 810 \\
  \textsc{triton} & 1606.71 & $6.4\times 10^{-10}$ & 1.000000 & 141 \\
  \textsc{linear scan} & 1606.71 & $1.2\times 10^{-9}$ & 1.000000 & 697 \\
\bottomrule
\end{tabular}
\captionsetup{justification=justified, singlelinecheck=false}
\caption{Full correctness validation of the Triton Semi-CRF implementation.
  \textbf{(a)}~Self-consistency invariants at large scale.
  \textbf{(b)}~Self-consistency invariants at genome scale.
  \textbf{(c)}~Finite-difference gradient validation: central differences ($\varepsilon{=}1.0\times 10^{-3}$) vs.\ autograd ($B{=}1$, $T{=}100$, $C{=}16$, $K{=}25$).
  \textbf{(d)}~Training convergence across backends (100 epochs, $B{=}4$, $T{=}500$, $C{=}32$, $K{=}50$). Performed on an NVIDIA L40S GPU.}
\label{tab:correctness-supp}
\end{table}

\clearpage
\section{Experiments}\label{sec:experiments}

We evaluate the ability of Flash-SemiCRF to perform exact semi-Markov inference across a range of values of $T$, $K$, and $C$, scaling from NER-sized problems to genome-scale sequences, measuring time, memory, and throughput. We further show that the natural candidates for exploiting bounded K -- banded or block-sparse matrix formats -- cannot prevent the near-dense intermediates that arise under composition in a parallel-scan formulation. Next, we benchmark linear CRF (K=1) and semi-CRF (K=30) on the DARPA TIMIT speech corpus~\cite{timit}. Before presenting scaling benchmarks, we address whether tree-structured alternatives could exploit the bounded K via banded kernels.

\subsection{Infeasibility of Banded Backends for Tree-Structured Inference}
\label{sec:banded_infeasibility}

Bounded segment length $K$ suggests that tree-structured backends could exploit sparsity via banded or block-sparse kernels. Two structural results show they cannot (full proofs in Appendix~\ref{sec:banded_is_not_viable}). First, at a tree node of span $S$, the feasible duration constraint $d_1 + d_2 \le S$ is \emph{anti-diagonal triangular}, not diagonal-banded; the clique formed by all states with $d \le \lfloor S/2 \rfloor$ forces a permutation-independent bandwidth lower bound of $C\lfloor S/2 \rfloor - 1$, so the compatibility matrix is near-dense once $S \gtrsim K$. Second, even genuinely banded one-step operators fill in under composition: $\mathrm{bw}(B^m) = \min(T, mK)$, saturating at dense width after $O(\log(T/K))$ tree levels. Empirically, bandwidth-reducing permutations yield ratios $\mathrm{bw}_{\mathrm{best}}/(n{-}1) \ge 0.97$ for $S \ge K$, and block-triangular formats that match the triangular geometry did not yield speedups due to indirect indexing overhead. Consequently, banded storage cannot prevent the near-dense intermediates that drive the high memory requirements in tree-structured formulations, motivating the streaming approach of Section~\ref{sec:streaming_algorithm}. Similarly, \texttt{pytorch-struct} implements a heuristic of $KC > 200$ state space switch to use the linear scan backend instead of the default tree search approach.

\subsection{Benchmarking}\label{sec:benchmarking}
\paragraph{Time and Memory}
We benchmark semi-CRF backends across $T$, $K$, and $C$ (Fig.~\ref{fig:benchmarks}). Tree-based methods (binary tree, memory sharded tree) achieve the lowest per-position latency at small state sizes but degrade rapidly as $(K{+}1)C$ increases (Fig.~\ref{fig:benchmarks}A,\,B), reflecting the cost of materializing $O(T(KC)^2)$ intermediates. Linear and streaming scans exhibit stable, state-size-independent compute scaling; the fused Triton kernel matches this stability while improving throughput over unfused linear scans through on-the-fly edge evaluation (Fig.~\ref{fig:benchmarks}E).

The dominant constraint at scale is memory, not compute. Linear and streaming scans maintain a small, constant memory footprint regardless of $T$, while tree-based and block-structured methods require orders of magnitude more memory (Fig.~\ref{fig:benchmarks}C). This separation is most visible in the out-of-memory (OOM) frontier (Fig.~\ref{fig:benchmarks}D,\,F): tree-based methods support progressively smaller state sizes as $T$ grows, whereas the fused kernel maintains a nearly constant maximum state size. Eliminating edge tensor materialization fundamentally alters the scaling behavior of semi-CRF inference.

To assess performance in regimes relevant to genomics-scale segmentation (large $T$ and $K$), we further evaluate the fused Triton implementation at extended sequence lengths and segment sizes (\ref{fig:benchmarks}G).In this setup, the tradeoff between memory and per-position runtime for K=2000–8000 is driven by hardware-specific constraints, likely due to L1 cache saturation, and will vary across systems. Further, in this setup where the competing methods are no longer tractable due to memory constraints, the fused kernel maintains near-linear runtime scaling in $K$ and a stable memory footprint, demonstrating practical applicability to kilobase-scale segmentation tasks.

\begin{figure}
        \centering
    \includegraphics[width=1\linewidth]{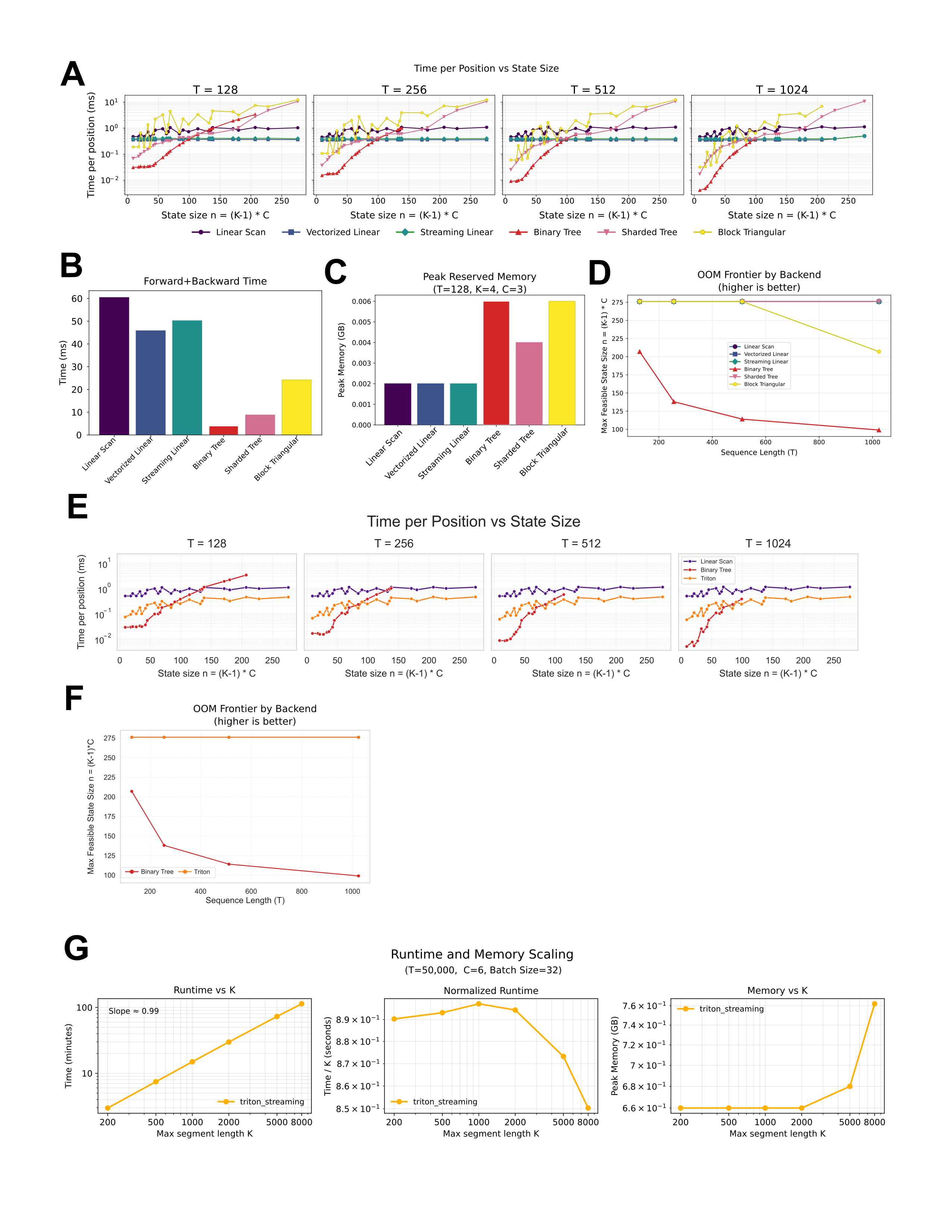}

        \caption{\textit{Theoretical Benchmarks}. \textbf{A.} Compute time per positions relationship to model state size between backends across different sizes T. \textbf{B.} Compute time for forward and backwards pass between each backend.   \textbf{C.} Peak Memory utilization between backends at T=128, K=4, C=3. \textit{D.} Mapping the out of memory (OOM) frontier for each backend across different sizes of T. \textbf{E.} Comparing fused Triton kernel optimization to original linear scan and binary tree backends (pytorch-struct) for compute time per position across different sizes T \textbf{F.} Demonstrating Triton implementation maintains linear memory requirements across sizes T while Binary tree is able to be used as a viable compute backend. \textbf{G.} Demonstration of triton streaming backend performance at scales T=50,000, C=6, batch size=32 and variable size K. } 
    \label{fig:benchmarks}
\end{figure}

\paragraph{DARPA TIMIT Speech Corpus}
To evaluate Flash-SemiCRF as a training and inference layer on a real segmentation task, we benchmark
on the DARPA TIMIT acoustic-phonetic corpus~\cite{timit}, a standard testbed for
structured sequence models in speech recognition. TIMIT is a natural fit for semi-CRFs because
phonemes have characteristic, label-dependent durations (vowels persist longer than stops, silence
segments are highly variable) that a duration-aware model can exploit. TIMIT provides
630 speakers (462 train, 168 test) with time-aligned phoneme annotations at 61 phone classes,
collapsed to the standard 39-phone set ($C = 39$, $C_\mathrm{PAD} = 64$)
following the Lee and Hon convention \cite{lee1989speaker}. Input features are 13 MFCCs with delta and delta-delta
coefficients (39 dimensions total) extracted at 10\,ms frame shift, z-score normalized using
training-set statistics.

\paragraph{Experimental setup.}
Both models share an identical 3-layer bidirectional LSTM encoder (hidden dimension 256, dropout 0.3)
followed by the Flash-SemiCRF structured decoder, trained with AdamW ($\eta = 10^{-3}$, weight
decay $10^{-5}$) and cosine annealing over 50 epochs (batch size 32). CRF parameters receive L2
regularization ($\lambda = 0.01$) and gradient clipping at norm 5.0. The partition function is
computed in float64 for numerical stability. Learnable scalar sequence boundary parameters
$\pi^{\mathrm{start}}, \pi^{\mathrm{end}} \in \mathbb{R}^C$ are folded into the cumulative scores
via the prefix-sum trick described in Appendix ~\ref{sec:boundaries}. The only difference between models is the
maximum segment duration: $K = 1$ (linear CRF) versus $K = 30$ (semi-CRF). This adds $K \times C =
1{,}170$ duration bias parameters to the semi-CRF, negligible relative to the encoder's ${\sim}1.6$M
parameters. At $C_\mathrm{PAD} = 64$, the per-checkpoint gradient workspace
(Section~\ref{sec:appendix_streaming}.4.4) requires ${\sim}120$\,MB for a batch of 32, well within the regime where workspace
memory is eclipsed by the edge tensor savings.

\paragraph{Convergence.}
Figure~\ref{fig:timit_convergence} shows training convergence across three metrics. Both models
converge to similar final Phone Error Rates (PER\,$= 0.218$ for $K{=}30$ vs.\ $0.219$ for $K{=}1$;
best PER $0.2178$ vs.\ $0.2187$ at epoch~43), consistent with the expectation that a strong
neural encoder absorbs much of the segmentation structure into its learned representations,
compressing the marginal contribution of explicit duration modeling to corpus-level PER. The
semi-CRF maintains a small but consistent advantage in boundary
F1 ($0.476$ vs.\ $0.468$) and segment F1 ($0.215$ vs.\ $0.207$), reflecting improved
segmentation coherence from explicit duration modeling.

The semi-CRF starts behind ($\mathrm{PER} = 0.95$ at epoch~0 vs.\ $0.89$ for $K{=}1$), requiring
additional warmup to explore the larger state space, but overtakes the linear CRF by epoch~5 on PER
and by epoch~1 on boundary and segment F1. The PER gap peaks near epoch~20
($\Delta \approx 0.005$) and compresses to ${\sim}0.001$ as both models converge, suggesting that
harder tasks with larger $T$ or $C$ may exhibit a more pronounced semi-CRF advantage.

Validation loss is consistently higher for the semi-CRF ($206.8$ vs.\ $183.7$ at epoch~50) despite
its lower PER. This is expected: the semi-CRF partition function sums over a combinatorially larger
space of valid $(K, C)$ segmentations, distributing probability mass more broadly than the per-frame
linear CRF. The discrepancy underscores that the negative log-likelihood and the MAP segmentation
accuracy measure distinct quantities.

\begin{table}[h]
    \centering
    \begin{tabular}{lcc}
        \toprule
        \textbf{Metric} & \textbf{$K{=}1$ (Linear CRF)} & \textbf{$K{=}30$ (Semi-CRF)} \\
        \midrule
        Phone Error Rate (PER) $\downarrow$       & 0.219 & \textbf{0.218} \\
        Boundary F1 $\uparrow$                    & 0.468 & \textbf{0.476} \\
        Boundary F1 $\pm$2 frames $\uparrow$           & 0.914 & \textbf{0.919} \\
        Segment F1 $\uparrow$                     & 0.207 & \textbf{0.215} \\
        \midrule
        Boundary entropy $H_{\mathrm{bdy}}$       & 5.687 (invariant) & \textbf{5.529} \\
        Position entropy $H_{\mathrm{pos}}$       & 3.651              & 3.651 \\
        Validation NLL                            & \textbf{183.7}     & 206.8 \\
        \bottomrule
    \end{tabular}
    \caption{TIMIT benchmark results ($C = 39$, 50 epochs, best checkpoint by PER at epoch~43).
    Both models share the same BiLSTM encoder (${\sim}1.6$M parameters); the semi-CRF adds
    $K \times C = 1{,}170$ duration bias parameters. Boundary entropy for the linear CRF is
    structurally fixed: $p(\mathrm{boundary} \mid t) = 1$ when $K{=}1$.}
    \label{tab:timit}
\end{table}

\begin{table}[t]
    \centering
    \begin{tabular}{llccccc}
        \toprule
        \textbf{Backend} & \textbf{Model} & $K$ & \textbf{Train (s/ep)} & \textbf{Infer (s)}
            & \textbf{utt/s} & \textbf{kfr/s} \\
        \midrule
        pytorch-struct linear scan  & Semi-CRF   & 30 & 4{,}430 & 1{,}243 & 0.8  & 0.3 \\
        Flash-SemiCRF (Triton)      & Semi-CRF   & 30 & 175     & 7       & 21.1 & 6.5 \\
        Flash-SemiCRF (PyTorch)     & Linear CRF &  1 & 27      & 32      & 140  & 42.8 \\
        \midrule
        \multicolumn{3}{l}{\textit{Speedup (Flash $K{=}30$ vs.\ pytorch-struct)}}
            & 25$\times$ & 178$\times$ & 26$\times$ & 22$\times$ \\
        \bottomrule
    \end{tabular}
    \caption{Wall-clock training and inference time on TIMIT (single NVIDIA L40S GPU,
    $B = 32$, $C = 39$). Flash-SemiCRF achieves a 25$\times$ training speedup and
    178$\times$ inference speedup over pytorch-struct's linear scan, the only existing
    backend that can run $K{=}30$ at this scale without OOM (tree backends exceed
    memory at $(K{+}1)C = 1{,}209$; see Fig.~\ref{fig:benchmarks}D). All three backends compute the same
    partition function (Table~\ref{tab:correctness-supp}d).}
    \label{tab:timit_timing}
\end{table}

\paragraph{Throughput.}
Table~\ref{tab:timit_timing} compares wall-clock training and inference time on a single L40S GPU.
The \texttt{pytorch-struct} linear scan~\cite{rush2020pytorchstruct} requires 4{,}430\,s per training epoch and
1{,}243\,s per inference pass at $K{=}30$, making a 50-epoch experiment a multi-day commitment.
Flash-SemiCRF reduces this to 175\,s per epoch (25$\times$ speedup) and 7\,s per inference pass
(178$\times$), processing 6.5k frames/s versus 0.3k.
Despite the $30\times$ increase in theoretical $O(TKC^2)$ complexity from $K{=}1$ to $K{=}30$, training the semi-CRF takes only $6.5\times$ longer per epoch: 175\,s versus 27\,s.

Inference time for $K{=}1$ (32\,s) is paradoxically slower than for $K{=}30$ (7\,s).
The $K{=}1$ CRF forward pass is fast (as reflected in the 27\,s training time, which is dominated
by forward-backward); the bottleneck is the unfused PyTorch Viterbi decode, which runs outside the
Triton kernel pipeline. The $K{\geq}3$ path dispatches to the fused Triton kernel for both
forward-backward and Viterbi (Appendix~\ref{sec:triton_k3}), avoiding this overhead. A fused $K{=}1$ Triton
Viterbi was not prioritized since the training loop, where kernel fusion provides the dominant
speedup, already dispatches efficiently at $K{=}1$.

\paragraph{Boundary entropy and structural differences in uncertainty.}
The clearest separation between models appears in boundary entropy (Figure~\ref{fig:timit_convergence}). Complementing the per-position label entropy $H_{\mathrm{pos}}$ (the sequence-averaged form of the frame-level class entropy used by Salvi~\cite{salvi2006boundary} for phonetic boundary detection), we define a sequence-level diagnostic as the Shannon entropy~\cite{shannon1948} of the normalized boundary distribution:
\begin{equation}
H_{\mathrm{bdy}} = -\sum_{t=1}^{T} \hat{p}_t \ln \hat{p}_t, \qquad \hat{p}_t = \frac{p(\mathrm{boundary} \mid t)}{\sum_{t'} p(\mathrm{boundary} \mid t')},
\end{equation}
where $p(\mathrm{boundary} \mid t)$ is the posterior segment-start probability from the forward-backward algorithm, and $\exp(H_{\mathrm{bdy}})$ gives the effective number of boundary positions. For $K{=}1$, every frame is its own segment, so $\hat{p}_t = 1/T$ and $H_{\mathrm{bdy}} = \ln T$ per utterance regardless of learned parameters; the linear CRF does not represent boundary uncertainty.

Averaged over the test corpus, the $K{=}1$ invariant is $\mathbb{E}[\ln T] = 5.687$ (geometric mean length $\approx 295$ frames). The semi-CRF reduces $H_{\mathrm{bdy}}$ from $5.671$ at initialization to $5.529$ by epoch~50, concentrating probability mass at true phoneme transitions; the effective boundary count $\exp(H_{\mathrm{bdy}})$ drops from ${\sim}295$ to ${\sim}252$. Per-position label entropy converges to the same value for both models ($H_{\mathrm{pos}} = 3.651$), indicating that the encoder representations are comparably expressive and the difference lies in the structural decoder.

\paragraph{Per-utterance uncertainty analysis.}
The corpus-level PER gap ($\Delta = 0.001$) averages over utterances where the models perform comparably and utterances where they diverge. Figure~\ref{fig:timit_uncertainty} shows posterior marginals $p(c \mid t)$ and boundary posteriors for DR4\_MPCS0\_SI729 ($T = 480$, 39 ground-truth segments), where the semi-CRF achieves PER\,$= 0.232$ versus the linear CRF's $0.411$, a 44\% relative reduction concentrated in regions of phonetic ambiguity. Both models activate similar label hypotheses, but the semi-CRF concentrates posterior mass more tightly along the correct label rows. The boundary posterior panel makes the structural difference explicit: the linear CRF assigns $p(\mathrm{boundary} \mid t) = 1.0$ uniformly (expected at $K{=}1$), while the semi-CRF's boundary posterior rises at true phoneme boundaries (red dashed lines) and falls to $p < 0.2$ in segment interiors. The semi-CRF thus provides a posterior over segmentation structure that a per-position model cannot represent, useful for downstream tasks requiring boundary uncertainty.

\begin{figure}[h]
    \centering
    \includegraphics[width=\textwidth]{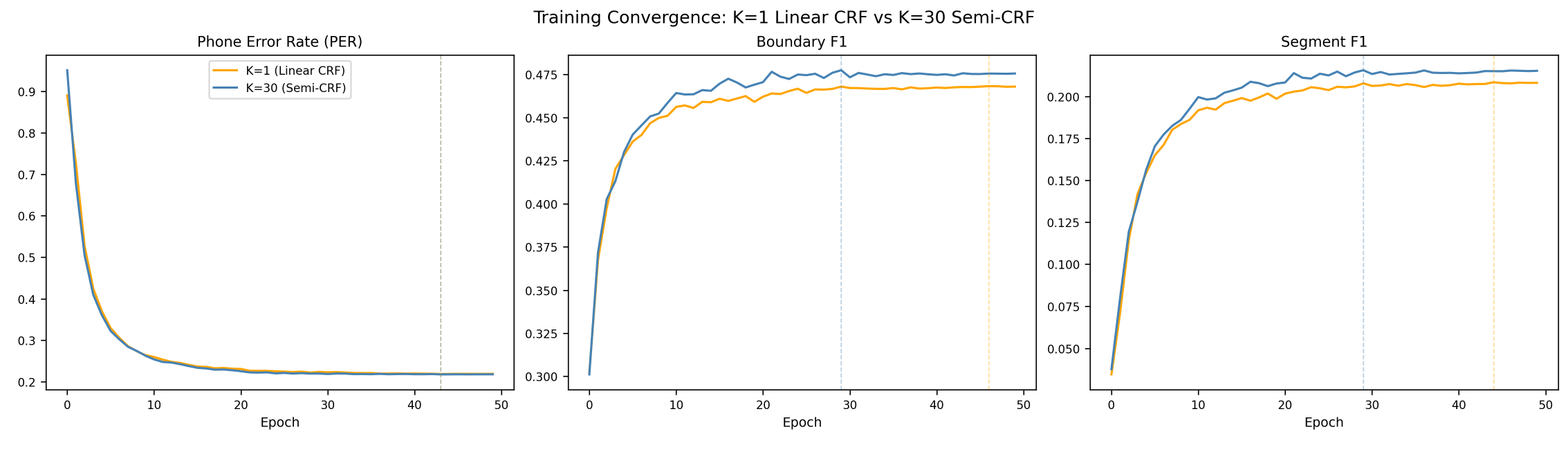}
    \caption{Training convergence on TIMIT: $K{=}1$ linear CRF vs.\ $K{=}30$ semi-CRF.
    Both models share an identical BiLSTM encoder; the only difference is the maximum
    segment duration. The semi-CRF maintains a similar Phone Error Rate (PER) as the linear CRF, but consistent improvement at
    boundary and segment F1 despite starting from a higher initial PER (larger state
    space requires additional warmup).}
    \label{fig:timit_convergence}
\end{figure}

\begin{figure}[t]
    \centering
    \includegraphics[width=\textwidth]{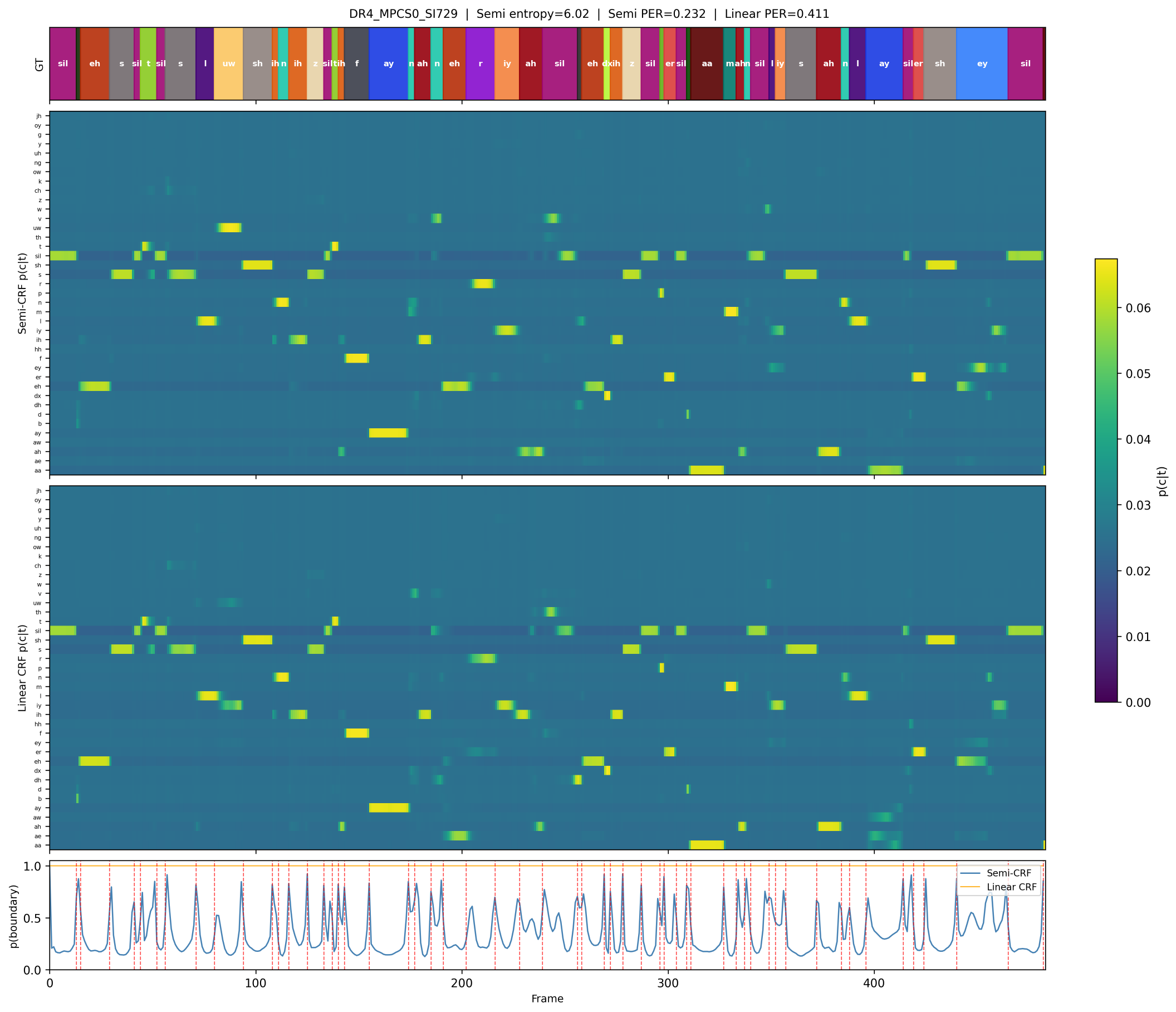}
    \caption{Per-utterance posterior analysis on TIMIT utterance DR4\_MPCS0\_SI729
    ($T = 480$ frames). \textbf{Top:} Ground-truth (GT) phoneme segmentation.
    \textbf{Middle panels:} Position-marginal heatmaps $p(c \mid t)$ for the
    semi-CRF ($K{=}30$, PER\,$= 0.232$) and linear CRF ($K{=}1$, PER\,$= 0.411$).
    Both models activate similar label hypotheses, but the semi-CRF concentrates
    posterior mass more tightly on correct labels.
    \textbf{Bottom:} Boundary posterior $p(\mathrm{boundary} \mid t)$.
    The semi-CRF (blue) varies across the utterance, peaking at true phoneme boundaries
    (red dashed lines) and suppressing in segment interiors. The linear CRF (orange)
    assigns $p(\mathrm{boundary} \mid t) = 1.0$ uniformly, the expected behavior when
    every frame is its own segment ($K{=}1$).}
    \label{fig:timit_uncertainty}
\end{figure}


\clearpage
\section{Discussion}\label{sec:discussion}

The central barrier to scaling semi-Markov Conditional Random Fields has not been their asymptotic complexity, but the memory constraints of the standard implementations. Existing approaches materialize an $O(TKC^2)$ edge tensor, making inference memory bound and infeasible on long sequences. By eliminating this tensor through prefix-sum decomposition and streaming dynamic programming, Flash-SemiCRF reframes semi-CRF inference as a compute-bound problem. This shift mirrors recent advances such as FlashAttention and scan-based sequence models, where avoiding intermediate materialization, not just changing asymptotics, enables orders of magnitude improvements in scale.

This change in computational approach has direct modeling consequences. Semi-CRFs, long limited to short sequences in domains such as natural language processing, now become viable as structured decoders for long-range sequence modeling. This addresses a persistent limitation of modern neural sequence models: while they learn rich contextual representations, they typically operate under a per-position prediction paradigm that does not enforce globally consistent segmentations. In contrast, semi-CRFs ensure coherent partitions of the input and model the label-dependent duration explicitly. With Flash-SemiCRF, these properties can now be applied at scale, enabling segmentation tasks to combine learned representations with explicit structural constraints (such as event detection, temporal action segmentation, or domain parsing). Both linear and semi-CRF models have complementary strengths that make them useful in different parts of a genomics workflow. Linear CRFs are computationally efficient and well-suited for annotating known sequence structures (e.g., Tranquillyzer, Bonito/Dorado). In contrast, semi-CRFs provide modeling of segment structure and boundary uncertainty, which can be advantageous for more complex annotation tasks, albeit at the cost of increased computational overhead due to the larger context space and exact marginals.

The edge potential decomposition (\ref{sec:boundary_projections}) also has consequences for gradient flow and encoder design. Because the boundary projections $\mathcal{P}^{\text{start}}$, $\mathcal{P}^{\text{end}}$ and the content emissions $f_\theta$ enter as separate additive terms, their gradients are structurally decoupled: the Jacobian is block-diagonal in ($W^{\text{start}}$, $W^{\text{end}}$) vs. $f_\theta$, so parameter updates can adjust boundary and content scoring independently. In a joint edge tensor parameterization these error signals are entangled, which can slow convergence when boundary and interior features have different characteristic scales or noise profiles. The additive decomposition also provides a natural interface for multi-modal encoders. Modalities that produce punctate, position-specific signals (e.g., Tn5 insertion densities from ATAC-seq) can be routed to the boundary projections via modality-specific heads on $\mathbf{h}_t$, while broad-domain signals (e.g., aggregate expression counts, histone mark coverage) feed the content term through $f_\theta$ and accumulate evidence via the prefix sum. This routing is possible precisely because the terms are precomputed independently; it would not be available in a parameterization where boundary and content features are jointly materialized.

Beyond computational efficiency, the formulation introduces a subtle but important statistical effect. The sequence-level centering of emissions induces an adaptive, data-dependent duration prior $\mathcal{B}^{\text{eff}}_{k,c} = \mathcal{B}_{k,c} - \nu_{b,c} \cdot k$. This acts as an implicit regularizer: labels with high average emissions incur stronger penalties for long segments, while low-prevalence labels receive a relative boost. In imbalanced settings, this mechanism suppresses degenerate segmentations dominated by frequent states and improves recovery of rare classes. At the same time, because this transformation is not path-invariant, it alters the underlying distribution and may be undesirable in settings requiring strict probabilistic interpretation. The choice between centered and uncentered formulations therefore reflects a trade-off between numerical stability, regularization, and model fidelity.

We note a limitation regarding the state space size $C$. To
ensure deterministic gradients and eliminate inter-segment atomic
contention on the GPU, Flash-SemiCRF allocates a gradient workspace of
size $O(B \cdot N_{\text{ckpt}} \cdot K \cdot C^2)$. While this
introduces a quadratic dependence on the number of labels, genomic
segmentation tasks, parsing promoters, UTRs, exons, introns, and
intergenic regions, rarely require $C > 64$. Within this regime, the
workspace memory remains low and is entirely eclipsed by the savings
from avoiding the $O(TKC^2)$ edge tensor materialization. The
implementation trades asymptotic scaling in $C$ for transformative
scaling in the dimensions that define genomic workloads ($T$ and $K$).

A more consequential constraint in practice is the maximum segment
duration $K$. Segments longer than $K$ cannot receive a unified duration
or content score and are instead captured through $K$-position lookback
and boundary projections alone, sacrificing the full segment-level
modeling that motivates the semi-CRF. In human genomics, median
protein-coding exons (${\sim}130$~bp) and most introns
(${\sim}1{,}539$~bp median) fall well within practical $K$ values, but
the 95th percentile intron length (${\sim}22{,}908$~bp) approaches the
regime where runtime and register pressure costs become significant
(Section~\ref{sec:loop_tiling}). Selecting $K$ to cover the segments of
greatest modeling interest is therefore the primary tuning decision, and
characterizing how segmentation quality degrades as features extend past
$K$ is an important direction for future work.

From a systems perspective, the results highlight a broader principle: structured models are often dismissed as unscalable due to nominal complexity, when in practice the limiting factor is memory traffic and tensor materialization. By converting semi-CRF inference into a fused, streaming kernel, we demonstrate that exact structured inference can achieve high throughput and full hardware utilization at sequence lengths exceeding $10^6$. This suggests that other structured models based on dynamic programming over variable-length segments, such as segmental
RNNs~\cite{kong2016segmental}, may admit similar reformulations.

\backmatter

\bmhead{Acknowledgements}

We gratefully acknowledge the feedback on this work from Brad Dickson. Computation for the work described in this paper was supported by the High Performance Cluster and Cloud Computing (HPC3) Resource at the Van Andel Research Institute.

\section*{Declarations}

\begin{itemize}
\item \textbf{Funding:} Support for this work was provided by VAI.
\item \textbf{Conflict of interest:} The authors declare no competing interests.
\item \textbf{Code availability:} [https://github.com/biobenkj/flash-semicrf]
\item \textbf{Author contribution:} B.K.J., T.G., A.S., H.S., and H.J.J. conceived the project. B.K.J. and T.G. implemented and bench marked the code base. B.K.J. and T.G. wrote the manuscript with contributions from A.S., H.S., and H.J.J.. H.S. and H.J.J. provided co-supervision and funding for this project.
\end{itemize}

\begin{appendices}
\numberwithin{equation}{section}
\numberwithin{table}{section}
\numberwithin{figure}{section}
\numberwithin{algocf}{section}

\section{Streaming Backend Supplement}\label{sec:appendix_streaming}
\let\origsection\section
\let\origsubsection\subsection
\let\origsubsubsection\subsubsection
\let\section\subsection
\let\subsection\subsubsection
\let\subsubsection\paragraph
%
%

\section{Notation}\label{supp:notation}

Table~\ref{tab:notation} collects all symbols used in the main text and supplement.

\begin{table}[ht]
\begin{tabular}{c l l}
\hline
\textbf{Symbol} & \textbf{Description} & \textbf{Shape} \\
\hline
$T$ & Padded sequence length (batch dimension) & scalar \\
$L_b$ & True (unpadded) length for batch element $b$ & scalar \\
$K$ & Maximum segment duration & scalar \\
$C$ & Number of labels (states) & scalar \\
$C_{\text{PAD}}$ & Padded label count (next power of 2) & scalar \\
$B$ & Batch size & scalar \\
$\mathcal{S}_{t,c}$ & Cumulative centered scores & $(B, T+1, C)$ \\
$f_\theta(t,c)$ & Raw projected encoder emissions & $(B, T, C)$ \\
$\nu_{b,c}$ & Sequence-level emission baseline & $(B, C)$ \\
$\bar{f}_\theta(t,c)$ & Centered emissions: $f_\theta(t,c) - \nu_{b,c}$ & $(B, T, C)$ \\
$\mathcal{T}_{c',c}$ & Transition scores (source $c'$ to dest $c$) & $(C, C)$ \\
$\mathcal{B}_{k,c}$ & Duration bias for duration $k$, label $c$ & $(K, C)$ \\
$\mathcal{P}^{\text{start}}_{t,c}$ & Start boundary projection (optional) & $(B, T, C)$ \\
$\mathcal{P}^{\text{end}}_{t,c}$ & End boundary projection (optional) & $(B, T, C)$ \\
$\tilde{\alpha}_t(c)$ & Log-forward message at position $t$, label $c$ & $(B, C)$ \\
$\tilde{\beta}_t(c)$ & Log-backward message & $(B, C)$ \\
$\boldsymbol{\alpha}$ & Forward ring buffer & $(B, K, C)$ \\
$\boldsymbol{\beta}$ & Backward ring buffer & $(B, 2K, C)$ \\
$\Omega$ & Checkpointed ring buffer states & $(B, N_{\text{ckpt}}, K, C)$ \\
$\mathcal{N}$ & Cumulative log-normalization factors & $(B, N_{\text{ckpt}})$ \\
$\Delta$ & Checkpoint interval ($\approx \sqrt{TK}$) & scalar \\
$\tau$ & Tile size for label dimension & scalar \\
$\log Z$ & Log partition function & $(B,)$ \\
\hline
\end{tabular}
\caption{Notation. Tilde ($\tilde{\cdot}$) denotes log-domain quantities. $N_{\text{ckpt}} = \lceil T / \Delta \rceil$ is the number of checkpoints.}
\label{tab:notation}
\end{table}

\textbf{Indexing conventions.}
Positions $t \in \{0, \ldots, T\}$ index \emph{boundaries} between tokens; cumulative scores $\mathcal{S}_{t,c}$ are defined at boundaries, with $\mathcal{S}_{0,c} = 0$.
Emissions $f_\theta(t, c)$ are defined at tokens $t \in \{0, \ldots, T{-}1\}$ (0-indexed).
Segments are half-open intervals $[s, s{+}k)$ covering tokens $s, s{+}1, \ldots, s{+}k{-}1$, so that the content score is $\mathcal{S}_{s+k, c} - \mathcal{S}_{s, c}$.
Durations are 1-indexed: $k \in \{1, \ldots, K\}$, stored at 0-indexed positions $k{-}1$ in code.

\section{Backward Pass with Checkpointing}\label{supp:backward}

The backward pass processes checkpoint segments in reverse order, recomputing $\boldsymbol{\alpha}$ from saved checkpoints following the gradient checkpointing strategy described in the main text. Algorithm~\ref{alg:streaming_backward} gives the full pseudocode. The saved normalizer $\mathcal{N}_i$ restores the true scale of $\boldsymbol{\alpha}$ when computing marginals.

\begin{algorithm}[t]
\caption{Streaming Semi-CRF Backward Pass}\label{alg:streaming_backward}
\KwIn{$\mathcal{S}$, $\mathcal{T}$, $\mathcal{B}$, $\mathcal{P}^{\text{start}}$, $\mathcal{P}^{\text{end}}$ (optional), checkpoints $(\Omega, \mathcal{N})$, $\log Z$, upstream $\partial \mathcal{L}/\partial Z$}
\KwOut{$\nabla \mathcal{S}$, $\nabla \mathcal{T}$, $\nabla \mathcal{B}$}
\BlankLine
$\boldsymbol{\beta} \gets -\infty \in \mathbb{R}^{2K \times C}$; \quad $\boldsymbol{\beta}[T \bmod 2K, :] \gets 0$\;
Initialize gradient accumulators\;
\BlankLine
\For{$i \gets N_{\text{ckpts}}-1$ \KwTo $0$}{
    $t_{\text{start}} \gets i \cdot \Delta$; \quad $t_{\text{end}} \gets \min((i+1) \cdot \Delta, T)$\;
    $\mathcal{N}_i \gets \mathcal{N}[i]$\;
    $\boldsymbol{\alpha}_{\text{local}} \gets$ \textsc{RecomputeAlpha}($\Omega[i]$, $t_{\text{start}}$, $t_{\text{end}}$)\;
    \BlankLine
    \For{$t \gets t_{\text{end}}-1$ \KwTo $t_{\text{start}}$}{
        $\tilde{\boldsymbol{\alpha}}_t \gets \boldsymbol{\alpha}_{\text{local}}[t - t_{\text{start}}, :]$\;
        $\tilde{\boldsymbol{\beta}}_t \gets -\infty$\;
        \BlankLine
        \For{$k \gets 1$ \KwTo $\min(K, T-t)$}{
            $\tilde{\boldsymbol{\beta}}_{\text{next}} \gets \boldsymbol{\beta}[(t+k) \bmod 2K, :]$\tcp*[r]{$2K$ ring buffer}
            $\tilde{\psi} \gets$ edge potential (main text, Eq.~\eqref{eq:streaming_potential})\;
            \BlankLine
            \tcp{Log-norm correction restores true $\alpha$ scale}
            $\log \mu \gets \tilde{\boldsymbol{\alpha}}_t[\text{None}, :] + \tilde{\psi} + \tilde{\boldsymbol{\beta}}_{\text{next}}[:, \text{None}] + \mathcal{N}_i - \log Z$\;
            $\mu \gets \exp(\log \mu)$\;
            \BlankLine
            Accumulate $\nabla \mathcal{S}$, $\nabla \mathcal{T}$, $\nabla \mathcal{B}$ from $\mu$\;
            $\tilde{\boldsymbol{\beta}}_t \gets \text{LSE}(\tilde{\boldsymbol{\beta}}_t, \text{LSE}(\tilde{\psi} + \tilde{\boldsymbol{\beta}}_{\text{next}}[:, \text{None}], \text{axis}=0))$\;
        }
        $\boldsymbol{\beta}[t \bmod 2K, :] \gets \tilde{\boldsymbol{\beta}}_t$\tcp*[r]{$2K$ ring buffer}
    }
}
\Return{$\nabla \mathcal{S}$, $\nabla \mathcal{T}$, $\nabla \mathcal{B}$}\;
\end{algorithm}

\section{Gradient Computation}\label{supp:gradients}

Gradients follow by the standard exponential-family identity: $\nabla_\theta \log Z$ equals the expected sufficient statistics under $p_\theta$, computed from the joint segment marginals $\mu(t,k,c,c') = \exp(\tilde{\alpha}_{t-k}(c') + \tilde{\psi} + \tilde{\beta}_t(c) + \mathcal{N}_i - \log Z)$ produced by the forward-backward pass. The implementation challenges lie not in these standard expressions but in accumulating them efficiently on GPU, which we address next.

\section{Reduced Atomics: Full Algorithms}\label{supp:reduced_atomics}

\subsection{Reduced Atomics Strategy}
GPU atomic operations are expensive and introduce non-determinism from floating-point non-associativity. For per-position gradients $\nabla \mathcal{S}_{t,c}$, we accumulate all $k$ and tile contributions into thread-local registers before a single atomic write, reducing atomics from $K \times \text{tiles}$ to 1 per position (e.g., $4{,}000 \to 1$ at $K{=}1000$, $C{=}64$). For shared parameters $\nabla \mathcal{T}$ and $\nabla \mathcal{B}$, we allocate per-checkpoint-segment workspace buffers that eliminate inter-segment atomic contention entirely, with a final deterministic \texttt{einsum} reduction on the host.

\subsection{Local Accumulation for Per-Position Gradients}

For $\nabla \mathcal{S}_{t,c}$, the negative contribution (from segments starting at $t$) is the same position for all $k$ values. Instead of $K \times \text{tiles}$ atomic operations, we accumulate locally:

\begin{algorithm}[H]
\caption{Local Accumulation for $\nabla \mathcal{S}_t$}\label{alg:local_accum}
$\texttt{grad\_cs\_t\_local} \gets \mathbf{0} \in \mathbb{R}^{C_{\text{PAD}}}$\;
\For{$k \gets 1$ \KwTo $K$}{
    \For{\text{each tile}}{
        $\texttt{grad\_cs\_t\_local} \mathrel{-}= \sum_{c'} \mu_{\text{tile}}(t, k, c, c')$\tcp*[r]{Register accumulation}
    }
}
$\texttt{atomic\_add}(\nabla \mathcal{S}_t, \texttt{grad\_cs\_t\_local})$\tcp*[r]{Single write}
\end{algorithm}

\textbf{Speedup}: $K \times \text{tiles} \to 1$ atomic per position (e.g., $1000 \times 4 = 4000 \to 1$ at $K=1000$, $C=64$).

\subsection{Per-Duration Accumulation for Duration Bias}

For $\nabla \mathcal{B}_{k,c}$, we accumulate across all tiles for each $k$, then write once:

\begin{algorithm}[H]
\caption{Per-Duration Accumulation for $\nabla \mathcal{B}_k$}
\For{$k \gets 1$ \KwTo $K$}{
    $\texttt{grad\_db\_k\_local} \gets \mathbf{0} \in \mathbb{R}^{C_{\text{PAD}}}$\;
    \For{\text{each tile}}{
        $\texttt{grad\_db\_k\_local} \mathrel{+}= \sum_{c'} \mu_{\text{tile}}(t, k, c, c')$\;
    }
    $\texttt{atomic\_add}(\nabla \mathcal{B}_k, \texttt{grad\_db\_k\_local})$\tcp*[r]{Once per $k$}
}
\end{algorithm}

\textbf{Speedup}: $\text{tiles} \to 1$ atomic per $(t, k)$ pair.

\subsection{Segment-Isolated Workspace Buffers}

For shared parameters ($\nabla \mathcal{T}$, $\nabla \mathcal{B}$), cross-segment atomic contention introduces non-determinism. We allocate per-segment workspace buffers:

\begin{equation}
\texttt{grad\_tr\_workspace} \in
    \mathbb{R}^{B \times N_{\text{segments}} \times K
               \times C_{\text{PAD}} \times C_{\text{PAD}}}
\end{equation}

\noindent (Padded to $C_{\text{PAD}}$ to prevent out-of-bounds access from masked Triton threads; sliced back to $C$ before reduction.)

Each checkpoint segment writes to its own slice, eliminating inter-segment atomics. The final reduction uses deterministic host-side operations:

\begin{align}
\nabla \mathcal{T}_{k,c',c}
    &= \texttt{einsum}(\text{``bskij, b} \to \text{kij''},
       \texttt{workspace.sum(dim=1)}, \nabla_{\text{out}}) \\
\nabla \mathcal{B}_{k,c}
    &= \texttt{einsum}(\text{``bskc, b} \to \text{kc''},
       \texttt{workspace.sum(dim=1)}, \nabla_{\text{out}})
\end{align}

\textbf{Determinism}: The segment-wise sum is deterministic (fixed order), and \texttt{einsum} performs a single reduction pass.

\textbf{Memory cost}: $O(B \cdot N_{\text{ckpt}} \cdot K \cdot C^2)$ workspace, which is acceptable since $N_{\text{segments}} \approx \sqrt{T/K}$ is typically small (e.g., 10--100 for $T=100\text{k}$, $K=1000$).

\section{Online Logsumexp for Tiled $\beta$ Reduction}\label{supp:online_lse}

The adaptive loop tiling described in the main text tiles the destination label dimension in blocks of size $\tau$. Since the $\beta$ reduction spans multiple tiles, we use the standard online logsumexp pattern from FlashAttention~\cite{dao2022flashattention} (tracking a running maximum and rescaled sum of exponentials across tile boundaries) to compute exact logsumexp without materializing the full $C_{\text{PAD}} \times C_{\text{PAD}}$ matrix.

\textbf{Two-pass tile normalization for marginals.} Marginal computation (Algorithm~\ref{alg:streaming_backward}, line~7) also tiles across destination labels. A naive approach would normalize each tile independently, but different tiles would use different local maxima, causing inconsistent marginals and 10--400\% gradient errors when $C > \tau$. The implementation uses a two-pass strategy: Pass~1 sweeps all tiles to compute a single global maximum $m^* = \max_{c} (\tilde{\boldsymbol{\alpha}}_t + \tilde{\psi} + \tilde{\boldsymbol{\beta}})$; Pass~2 reprocesses all tiles using $m^*$ as a fixed normalizer, ensuring every tile produces marginals on the same scale.

\section{Additional Numerical Stability Measures}\label{supp:stability}

Beyond emission baseline centering and checkpoint boundary normalization (both described in the main text), the implementation employs several additional stability measures.

\textbf{Log-domain.} All forward-backward computations use logsumexp: $\text{LSE}(\mathbf{x}) = \max(\mathbf{x}) + \log\sum_i \exp(x_i - \max(\mathbf{x}))$.

\textbf{NEG\_INF guards.} When all logsumexp inputs are $-10^9$, the subtraction $x - \max(x) = 0$ instead of remaining at $-\infty$. Guards detect this case ($\max < -10^9 + 1$) and return $-\infty$ directly.

\textbf{Float64 accumulation.} Gradient tensors for shared parameters use float64 to prevent non-determinism from \texttt{atomic\_add} floating-point non-associativity. Error scales as $O(\sqrt{T \times K \times C})$ per operation; float64 reduces this from ${\sim}10^{-3}$ (float32) to ${\sim}10^{-10}$ (negligible).

\textbf{Log-scale clamping.} The log-scale factor $\log s = \max(\boldsymbol{\alpha}) + \mathcal{N}_i - \log Z$ is clamped to $[-700, 0]$ before $\exp()$. The upper bound of $0$ reflects that segment marginals are probabilities ($s \le 1$); the lower bound of $-700$ prevents float64 underflow ($\exp(-710) \approx 0$).

\textbf{Intermediate clamping.} Before computing the joint log-probability $\tilde{\boldsymbol{\alpha}}_t + \tilde{\psi} + \tilde{\boldsymbol{\beta}}$, each term is clamped to $[-10^6, 10^6]$. This prevents float overflow when three large-magnitude terms are summed, without affecting valid values (which are $O(\sqrt{T})$ after checkpoint normalization).

\textbf{Masking.} Invalid positions use $-10^9$ (not $-\infty$) to avoid NaN in gradients.

\textbf{Variable-length batches.} For sequences ending at $L < T$, the normalization shift is masked to zero for $t > L$, freezing $\mathcal{N}_{\text{accum}}$ at the correct value.

\section{Emission Baseline: Extended Notes}\label{supp:centering_notes}

\textbf{Padding semantics.}
The implementation computes the baseline using a \emph{masked mean} over valid positions only:
\begin{equation}
\nu_{b,c} = \frac{1}{L_b} \sum_{u=0}^{L_b-1} f_\theta(u, c)
\end{equation}
where $L_b$ is the true (unpadded) sequence length for batch element $b$.  Padding tokens ($u \ge L_b$) do not influence the baseline, so the effective model is independent of padding scheme and padded length $T$.  For the primary use case of fixed-length genomic windows ($L_b = T$ for all batch elements), this reduces to the ordinary mean over the full sequence.

\textbf{Semantics-preserving alternative.}
A semantics-preserving alternative exists: build prefix sums on the centered residuals $\bar{f}_\theta(t,c)$ but reconstruct the original segment score at evaluation time via $(\mathcal{S}_{s+k,c} - \mathcal{S}_{s,c}) + \nu_{b,c} \cdot k$.  This retains the numerical benefit ($O(\sqrt{T})$ cumulative sum magnitude) without altering the semi-CRF distribution.  We use the globally-centered variant instead for three reasons: (1) the adaptive prior provides empirically useful regularization in label-imbalanced genomic sequences; (2) the learned duration bias $\mathcal{B}_{k,c}$, with $K \times C$ free parameters, can absorb the average effect of $\nu_{b,c}$, so the prior primarily regularizes per-sample variation rather than imposing a fixed bias (parametric duration models with $O(C)$ parameters cannot represent the linear-in-$k$ correction and likely require the reconstruction variant instead); and (3) the reconstruction variant requires propagating $\nu_{b,c}$ into the Triton kernel as an additional argument, adding complexity for marginal benefit in the fixed-length-window regime where $L_b = T$.

\section{Implementation Correspondence}\label{supp:implementation}

Table~\ref{tab:implementation} maps mathematical notation to code variables for reference.

\begin{table}[ht]
\begin{tabular}{l l l}
\hline
\textbf{Math} & \textbf{Code Variable} & \textbf{Notes} \\
\hline
$\mathcal{S}_{t,c}$ & \texttt{cum\_scores[:, t, c]} & Built from centered emissions \\
$f_\theta(t,c)$ & \texttt{scores} / \texttt{projected} & Raw encoder output \\
$\nu_{b,c}$ & \texttt{\_center\_scores(scores, lengths)} & Masked mean over valid $L_b$ \\
$\mathcal{T}_{c',c}$ & \texttt{transition[c\_src, c\_dst]} & Source-first storage \\
$\mathcal{B}_{k,c}$ & \texttt{duration\_bias[k-1, c]} & 0-indexed in code \\
$\boldsymbol{\alpha}$ & \texttt{ring\_buffer} / \texttt{alpha\_ring} & Size $K$ \\
$\boldsymbol{\beta}$ & \texttt{beta\_ring} & Size $2K$ \\
$\mathcal{P}^{\text{start}}$ & \texttt{proj\_start} & Optional, $(B, T, C)$ \\
$\mathcal{P}^{\text{end}}$ & \texttt{proj\_end} & Optional, $(B, T, C)$ \\
$\Omega$ & \texttt{ring\_checkpoints} & Saved at intervals \\
$\mathcal{N}$ & \texttt{log\_norm\_checkpoints} & Cumulative log-norm \\
$\Delta$ & \texttt{checkpoint\_interval} & $\approx \sqrt{TK}$ \\
$C_{\text{PAD}}$ & \texttt{C\_PAD} & \texttt{\_next\_power\_of\_2(C)} \\
$\tau$ & \texttt{TILE\_C} & Adaptive \\
-- & \texttt{grad\_tr\_workspace} & $(B, N_{\text{ckpt}}, K, C_{\text{PAD}}, C_{\text{PAD}})$ \\
-- & \texttt{grad\_db\_workspace} & $(B, N_{\text{ckpt}}, K, C_{\text{PAD}})$ \\
\hline
\end{tabular}
\caption{Mapping between mathematical notation and implementation. $N_{\text{ckpt}}$ denotes number of checkpoint segments.}
\label{tab:implementation}
\end{table}

\let\section\origsection
\let\subsection\origsubsection
\let\subsubsection\origsubsubsection

\section{Semi-Markov CRF Backend Algorithms}\label{sec:appendix_backends}
\let\origsection\section
\let\origsubsection\subsection
\let\origsubsubsection\subsubsection
\let\section\subsection
\let\subsection\subsubsection
\let\subsubsection\paragraph

This supplement summarizes the semi-CRF inference backends benchmarked in the main text, with complexity analysis and provenance for each.

\section{Notation}

All notation follows the main text; see also the notation table in Appendix~\ref{sec:appendix_streaming}. We additionally use $n = (K+1) \cdot C$ for the expanded state-space size and $\mathbf{W}_t$ for the transition matrix at position $t$ in tree-structured backends.

\paragraph{Important Distinction: Edge Tensor Materialization.}
All backends in this supplement assume the edge potential tensor $\psi \in \mathbb{R}^{B \times (T-1) \times K \times C \times C}$ is \textbf{pre-materialized} before inference. This $O(TKC^2)$ tensor dominates memory for large $T$, regardless of which backend processes it. True $T$-independent memory--where edge potentials are computed on-the-fly from $O(TC)$ cumulative scores--is achieved only in the \textbf{fused Triton kernel} (Appendix~\ref{sec:appendix_streaming}).


\section{Semiring Abstraction}
\label{sec:backend_semirings}

All backends use the standard semiring dynamic programming framework~\cite{goodman1999semiring}: substituting the $(\oplus, \otimes)$ operations yields the partition function (Log semiring), Viterbi path (Max semiring), $k$-best lists, samples, or entropy from the same algorithm; see Rush~\cite{rush2020pytorchstruct}, Table~2 for the full catalog of semirings and their gradient interpretations. For large state spaces, we provide checkpointing wrappers that trade $2\times$ compute for reduced memory~\cite{chen2016training}.

\section{Provenance}
\label{sec:backend_provenance}

\begin{table}[h]
\begin{tabular}{@{}ll@{}}
\toprule
\textbf{Backend} & \textbf{Source} \\
\midrule
Binary tree (dense) & pytorch-struct~\cite{rush2020pytorchstruct} \\
Binary tree (sharded) & pytorch-struct~\cite{rush2020pytorchstruct} \\
Linear scan (standard) & pytorch-struct~\cite{rush2020pytorchstruct} \\
\midrule
Linear scan (vectorized) & This work \\
Streaming linear (PyTorch) & This work \\
Block-triangular & This work \\
Banded (analysis) & This work \\
\midrule
Streaming (Triton kernel) & This work \\
\bottomrule
\end{tabular}
\caption{Algorithm provenance. The Triton kernel (bottom) is the key contribution: it computes edge potentials on-the-fly, achieving true $T$-independent memory and enabling genome-scale inference.}
\end{table}

\section{Backend 1: Binary Tree (Parallel Scan)}
\label{sec:backend_binary_tree}

The binary tree algorithm is the default semi-Markov implementation in \texttt{pytorch-struct}~\cite{rush2020pytorchstruct}. It reformulates the forward recursion as semiring matrix products over an expanded state space of size $n = (K+1) \cdot C$, where each state $(k, c)$ represents ``$k$ time steps remaining in a segment of label $c$'' (see Rush~\cite{rush2020pytorchstruct} for the full transition matrix construction). The forward pass $\mathbf{A}_{1:T} = \mathbf{W}_1 \otimes \cdots \otimes \mathbf{W}_T$ is computed via binary tree reduction in $O(\log T)$ parallel depth.

\paragraph{Complexity.} Time: $O(T(KC)^3)$. Space: $O(T(KC)^2)$. The quadratic memory in $KC$ causes OOM before the depth advantage is realized. A sharded variant~\cite{rush2020pytorchstruct} adds gradient checkpointing~\cite{chen2016training} to reduce memory to $O(\sqrt{T}(KC)^2)$ with $2\times$ compute overhead, extending the viable regime from $n < 100$ to $n < 150$.

\section{Backend 2: Block-Triangular Sparsity (Experimental)}
\label{sec:backend_block_tri}

At a tree node of span $S$, the duration constraint $k_1 + k_2 \leq S$ means only $\frac{K(K+1)}{2}$ duration pairs are feasible, yielding block-triangular sparsity.

\paragraph{Representation.} Store only blocks $(k_1, k_2)$ satisfying $k_1 + k_2 \leq S$, each block being dense $C \times C$. This halves storage vs. dense at relevant spans.

\paragraph{Sparse Multiplication.} For block-triangular matrices $\mathbf{C}, \mathbf{D}$:
\begin{equation}
\mathbf{E}[k_1, k_3] = \bigoplus_{k_2 : k_1 + k_2 \leq S,\, k_2 + k_3 \leq S} \mathbf{C}[k_1, k_2] \otimes \mathbf{D}[k_2, k_3]
\end{equation}

\paragraph{Practical Outcome.} In PyTorch/CUDA experiments, block-triangular storage did \textbf{not} yield speedups at typical sparsity levels. The overhead of indirect indexing and non-contiguous memory access outweighed savings from skipping $\sim$50\% of blocks. GPUs prefer well-tuned dense kernels unless sparsity exceeds $\sim$90\%.


\section{Backend 3: Linear Scan}
\label{sec:backend_linear_scan}

The standard linear scan is the reference semi-Markov implementation in \texttt{pytorch-struct}~\cite{rush2020pytorchstruct}. It implements the textbook semi-CRF forward recursion $\tilde{\alpha}_t(c) = \text{LSE}_{k,c'}[\tilde{\alpha}_{t-k}(c') + \tilde{\psi}(t, k, c, c')]$ sequentially over positions, avoiding the $O((KC)^2)$ intermediate storage of tree backends. A vectorized variant replaces inner loops with batched tensor operations, yielding 2--3$\times$ speedup.

\paragraph{Complexity.} Time: $O(TKC^2)$. Space: $O(TC)$ for forward messages, or $O(KC)$ with ring buffer; in both cases the pre-materialized edge tensor adds $O(TKC^2)$, which dominates. Depth: $O(T)$.

\section{Backend 4: Streaming Linear Scan}
\label{sec:backend_streaming}

The streaming linear scan reduces \emph{forward message} memory from $O(TC)$ to $O(KC)$ via a ring buffer ($\tilde{\alpha}_{t}(c) = \boldsymbol{\alpha}[t \bmod K, c]$). However, this PyTorch implementation still requires the pre-materialized edge tensor as input.

\paragraph{Complexity.} Time: $O(TKC^2)$. Space: $O(TKC^2)$ for edge tensor + $O(KC)$ for ring buffer. True $T$-independent memory requires computing edges on-the-fly via the prefix-sum decomposition (main text, Eq.~\eqref{eq:streaming_potential}), achieved only in the fused Triton kernel (Appendix~\ref{sec:appendix_streaming}).

\section{Complexity Summary}

\begin{table}[h]
\begin{tabular}{@{}lcccl@{}}
\toprule
Backend & Work & Memory & Depth & Notes \\
\midrule
Binary tree (dense) & $O(Tn^3)$ & $O(Tn^2)$ & $O(\log T)$ & $n < 100$ \\
Binary tree (sharded) & $O(Tn^3)$ & $O(\sqrt{T}n^2)$ & $O(\sqrt{T})$ & $n < 150$ \\
Block-triangular & $O(Tn^3)$ & $\sim\frac{1}{2}O(Tn^2)$ & $O(\log T)$ & Not faster \\
Banded & $O(Tn^3)$ & $\to O(Tn^2)$ & $O(\log T)$ & Not viable \\
Linear scan & $O(TKC^2)$ & $O(TKC^2)^*$ & $O(T)$ & General \\
Streaming linear (PyTorch) & $O(TKC^2)$ & $O(TKC^2)^*$ & $O(T)$ & Ring buffer for $\alpha$ \\
\midrule
Streaming (Triton kernel) & $O(TKC^2)$ & $O(\sqrt{T/K} \cdot KC)^{\dagger}$ & $O(T)$ & \textbf{Genome-scale} \\
\bottomrule
\end{tabular}
\caption{Backend comparison. $n = (K+1) \cdot C$. $^*$Dominated by pre-materialized edge tensor. $^{\dagger}$Working memory is $O(KC)$ (ring buffers); checkpoint memory grows sublinearly as $O(\sqrt{T/K} \cdot KC)$. Edges computed on-the-fly; see Appendix~\ref{sec:appendix_streaming}.}
\end{table}

\section{Duration Distributions}
\label{sec:backend_duration}

The duration bias $\mathcal{B}_{k,c}$ can be learned as $K \times C$ free parameters or constrained to standard parametric forms (Geometric, Negative Binomial, Poisson, Uniform), reducing the parameter count to $O(C)$ and encoding inductive bias about expected segment lengths.


\clearpage

\let\section\origsection
\let\subsection\origsubsection
\let\subsubsection\origsubsubsection

\section{Infeasibility of Banded Backends for Tree-Structured Semi-CRF Inference}\label{sec:banded_is_not_viable}
%

This appendix analyzes a natural optimization for exact semi-CRF~\cite{sarawagi2004semicrf} inference: using banded or block-sparse matrix kernels inside a binary-tree (parallel-scan) formulation of the dynamic program. The conclusion is negative: within this formulation, banded storage can reduce cost only at the smallest spans, but cannot prevent near-dense intermediate operators at the levels that dominate runtime and memory.

\textbf{Scope.} The claims below apply to \emph{exact} tree-structured implementations that represent node-to-node combination as semiring matrix products over an expanded state space of size $\Theta(KC)$. This does not rule out approximate inference (e.g., pruning) or alternative exact formulations that do not materialize these operators.

\subsection{Notation}\label{supp:banded_notation}

We consider a sequence of length $T$, equivalently boundary indices $0,1,\dots,T$. The maximum segment length is $K$, and the label/state count is $C$.

Two matrices are used as simplifying models:

\textbf{Boundary reachability (global boolean model).} Define $B \in \{0,1\}^{(T+1)\times(T+1)}$ by
\begin{equation}
B[i,j] = \begin{cases} 1 & \text{if } 1 \leq j - i \leq K \\ 0 & \text{otherwise.} \end{cases}
\end{equation}
This captures the monotone-time constraint ``a single segment advances forward by at most $K$ boundaries.'' For instance, with $T = 5$ and $K = 2$, the matrix $B$ has $B[0,1] = B[0,2] = B[1,2] = B[1,3] = \cdots = 1$, but $B[0,3] = 0$ (since $3 - 0 = 3 > K$).

\textbf{Duration compatibility at a tree node (local pattern).} At a tree node of span length $S$, partial states are indexed by a duration $d \in \{1,\dots,D\}$ where $D=\min(K,S)$, and a label $y\in\{1,\dots,C\}$. The quantity $D$ caps the duration states at the actual feasible range for the current span: durations cannot exceed the span $S$, nor can they exceed the global maximum $K$. Feasible left/right duration pairs satisfy $d_1+d_2 \le S$. For each feasible duration pair, label interactions are typically dense, yielding \emph{block-dense} sparsity.

\paragraph{Effective bandwidth.}
For a square matrix $M$ whose indices have been flattened to $\{1,\dots,n\}$, define the (structural) bandwidth
\begin{equation}
\mathrm{bw}(M) \;=\; \max\{\,|i-j| \;:\; M[i,j] \text{ is structurally nonzero}\,\}.
\end{equation}
For boolean matrices, ``nonzero'' means $1$; for semiring score matrices, it means ``not masked / not $-\infty$.'' A banded representation with half-bandwidth $b$ requires $\mathrm{bw}(M)\le b$ (after any chosen permutation).

\subsection{Local Incompatibility Pattern Is Triangular, Not Diagonal-Banded}

At an internal tree node of span $S$, a standard combine operation must account for left and right partial durations $d_1,d_2$ that together fit inside the span:
\begin{equation}
d_1 + d_2 \le S.
\end{equation}
In the $(d_1,d_2)$ plane, this feasible set is an anti-diagonal triangle. A diagonal band, by contrast, corresponds to $|d_1-d_2|\le b$. These shapes overlap only partially, and the mismatch becomes pronounced as soon as $S$ is a nontrivial fraction of $K$.

To make ``not narrow-banded under any ordering'' precise, it suffices to lower bound the minimum achievable bandwidth.

\begin{proposition}[Bandwidth lower bound from a clique]\label{prop:bw_lower_bound}
Fix $S$, let $D=\min(K,S)$, and define a binary matrix $M_S \in \{0,1\}^{(DC)\times(DC)}$ indexed by pairs $(d,y)\in\{1,\dots,D\}\times\{1,\dots,C\}$, with
\begin{equation}
M_S[(d_1,y_1),(d_2,y_2)] = 1 \quad \Longleftrightarrow \quad d_1 + d_2 \le S.
\end{equation}
For any simultaneous row/column permutation $\pi$,
\begin{equation}
\mathrm{bw}\!\left(P_\pi^\top M_S P_\pi\right) \;\ge\; C\Big\lfloor \frac{S}{2}\Big\rfloor - 1.
\end{equation}
\end{proposition}

\begin{proof}
Let $m=\lfloor S/2\rfloor$. For any $d,d'\le m$, $d+d'\le 2m\le S$, so all states in
\begin{equation}
U=\{1,\dots,m\}\times\{1,\dots,C\}
\end{equation}
are mutually adjacent in the undirected graph induced by $M_S$; hence $U$ is a clique of size $|U|=mC$. For any linear ordering of a clique on $n$ vertices, the first and last vertices are adjacent, so the bandwidth is at least $n-1$.
\end{proof}

\textbf{Consequence.} When $S$ is moderate (e.g., $S\approx K$) and $C$ is not tiny, the best achievable bandwidth is already a large fraction of the dense width $DC$. In the extreme case $S \ge 2K$--since $D = K$ when $S \ge K$, the condition $d_1 + d_2 \le S$ with $d_1, d_2 \le K$ is always satisfied when $S \ge 2K$--the matrix $M_S$ is structurally dense.

\subsection{Fill-In Under Composition in a Boolean Reachability Model}

Even if the local operator were genuinely banded, repeated composition widens the support.

\begin{lemma}[Boolean powers widen linearly]\label{lem:boolean_powers}
Let $B$ be the boundary adjacency matrix defined above, and interpret multiplication over the boolean semiring (\textup{OR} as addition, \textup{AND} as multiplication). Then $(B^m)[i,j]=1$ if and only if there exists a path of exactly $m$ segments from $i$ to $j$. For $m\ge 1$,
\begin{equation}
(B^m)[i,j]=1 \quad \Longleftrightarrow \quad m \le j-i \le mK,
\end{equation}
and therefore $\mathrm{bw}(B^m) = \min(T,\, mK)$.
\end{lemma}

\begin{proof}
A path of $m$ segments advances by $\delta_1+\cdots+\delta_m$ with each $\delta_t\in\{1,\dots,K\}$, hence $m \le j-i \le mK$. Conversely, if $r=j-i$ satisfies $m\le r \le mK$, then $r$ can be written as a sum of $m$ integers in $[1,K]$ (e.g., start from all ones and distribute the remaining $r-m$ units without exceeding $K$).
\end{proof}

\paragraph{Implication for a balanced binary tree.}
In a balanced binary tree~\cite{blelloch1990prefix}, combining two children corresponds (in this simplified reachability model) to composing reachability over a larger number of segments. The relevant parameter $m$ therefore grows geometrically with tree level, so $\mathrm{bw}(B^m)$ increases rapidly and saturates at the dense width on the order of $\log_2(T/K)$ levels. The reachability pattern becomes near-dense well before the root, so a banded container offers little benefit at the levels that dominate the total work.

\subsection{Why Matrix Reordering Helps Only at Very Small Spans}

A reasonable countermeasure is to apply bandwidth-reducing permutations (e.g., Reverse Cuthill--McKee~\cite{cuthill1969bandwidth}; see also~\cite{george1981sparse}). Proposition~\ref{prop:bw_lower_bound} already limits what is possible at a single node: cliques force large bandwidth regardless of ordering.

At the global scale, once composition produces near-complete reachability, permutations cannot help substantially. One clean sufficient condition is ``undirected completeness'': if for every unordered pair of indices $\{u,v\}$ at least one of $M[u,v]$ or $M[v,u]$ is structurally nonzero, then the induced undirected graph is a clique and the bandwidth is $n-1$ under any ordering.

\paragraph{Empirical bandwidth ratios.}
The accompanying bandwidth report (available on github with the codebase) evaluates several orderings across $(K,C,S)$. Summarizing $\mathrm{bw}_\text{best}/(n-1)$ by span fraction $S/K$: for small spans ($S \le K/2$), best-case bandwidth ratios ranged from $\approx 0.23$ to $\approx 0.68$ (mean $\approx 0.44$); for moderate spans ($K/2 < S \le K$), ratios ranged from $\approx 0.78$ to $\approx 0.99$ (mean $\approx 0.90$); for large spans ($S \ge K$), ratios were $\ge 0.97$ and often exactly $1.0$. Thus, reordering can reduce bandwidth when spans are very small, but offers negligible benefit once spans approach $K$, the regime that dominates higher tree levels.

\subsection{Practical Implications for Exact Semi-Markov Backends}

The computation and memory of tree-structured backends are dominated by higher tree levels, where local duration compatibility is near-dense (Proposition~\ref{prop:bw_lower_bound}) and multi-step composition produces near-dense support (Lemma~\ref{lem:boolean_powers}). Block-triangular formats that match the triangular geometry reduce storage by roughly half, but in our experiments, indirect indexing overhead negated this at typical sparsity levels, consistent with the well-known GPU performance pattern that moderate block sparsity underperforms dense kernels unless sparsity exceeds ${\sim}90$\%.

If sparsity is essential, it must be introduced by approximation (e.g., segment filtering~\cite{zaratiana2023fsemicrf} or hybrid architectures~\cite{ye2018hybrid,kong2016segmental}), which enforce value sparsity rather than relying on structural sparsity that does not persist under composition.

\section{Linear and Near-Linear CRF Implementations}\label{sec:appendix_linear_crf}
\let\origsection\section
\let\origsubsection\subsection
\let\origsubsubsection\subsubsection
\let\section\subsection
\let\subsection\subsubsection
\let\subsubsection\paragraph
%

This supplement documents the specialized $K=1$ (linear CRF) and $K=2$ (near-linear CRF) implementations in \texttt{flash-semicrf}. These optimized paths eliminate ring buffer overhead for common sequence labeling tasks while maintaining compatibility with the streaming Semi-CRF interface.

\section{Overview}

The streaming Semi-CRF module automatically dispatches to specialized implementations based on the maximum segment duration $K$:

\begin{table}[h]
\centering
\begin{tabular}{@{}lll@{}}
\toprule
$K$ & Implementation & Rationale \\
\midrule
$K=1$ & \texttt{LinearCRFStreaming} & Standard linear-chain CRF; no duration loop \\
$K=2$ & \texttt{SemiCRFK2Streaming} & Explicit 2-step history; avoids ring buffer edge cases \\
$K \geq 3$ & \texttt{SemiCRFStreamingTriton} & Full ring buffer architecture \\
\bottomrule
\end{tabular}
\caption{Automatic dispatch by maximum segment duration. When full boundary projections ($\mathcal{P}^{\text{start}}$, $\mathcal{P}^{\text{end}}$) are active, $K=1$ and $K=2$ dispatch falls through to the generic $K \geq 3$ streaming path. Scalar boundaries ($\pi^{\text{start}}$, $\pi^{\text{end}}$) are folded into $\mathcal{S}$ and do not affect dispatch.}
\end{table}

\paragraph{Triton Kernel Scope.} The specialized $K=1$ and $K=2$ paths are implemented only in PyTorch. The Triton streaming kernel requires $K \geq 3$ for correct ring buffer operation. This is intentional: the PyTorch implementations are already efficient for small $K$, and the Triton kernel's complexity is justified only when ring buffers and checkpointing provide substantial memory savings.

\section{Notation}

\begin{center}
\begin{tabular}{c l l}
\toprule
\textbf{Symbol} & \textbf{Description} & \textbf{Shape} \\
\midrule
$T$ & Sequence length & scalar \\
$C$ & Number of labels (states) & scalar \\
$B$ & Batch size & scalar \\
$\mathcal{S}_{t,c}$ & Cumulative projected scores & $(B, T+1, C)$ \\
$\mathcal{T}_{c',c}$ & Transition scores (source $c'$ to dest $c$) & $(C, C)$ \\
$\mathcal{B}_{k,c}$ & Duration bias for duration $k$, label $c$ & $(K, C)$ \\
$\tilde{\alpha}_t(c)$ & Log-forward message at position $t$, label $c$ & $(B, C)$ \\
$\tilde{\beta}_t(c)$ & Log-backward message & $(B, C)$ \\
$\pi_c^{\text{start}}$ & Learnable scalar start boundary & $(C,)$ \\
$\pi_c^{\text{end}}$ & Learnable scalar end boundary & $(C,)$ \\
$\mathcal{P}^{\text{start}}_{t,c}$ & Full start boundary projection & $(B, T, C)$ \\
$\mathcal{P}^{\text{end}}_{t,c}$ & Full end boundary projection & $(B, T, C)$ \\
\bottomrule
\end{tabular}
\captionof{table}{Notation. Tilde ($\tilde{\cdot}$) denotes log-domain quantities. Durations are 1-indexed in the equations: $k \in \{1, \ldots, K\}$, stored at 0-indexed positions $k{-}1$ in the implementation.}
\end{center}

\section{Score Centering}
\label{sec:appendix_centering}

Score centering follows the main text
(Section~\ref{sec:cumulative_scores}): the per-sequence baseline
$\nu_{b,c}$ is subtracted from raw projected scores before constructing
the cumulative sum $\mathcal{S}$, controlling cumulative magnitude from
$O(T)$ to $O(\sqrt{T})$. For variable-length batches, the baseline is
computed as a masked mean over the $L_b$ valid positions (see
Appendix~\ref{sec:appendix_streaming} for padding semantics).

\paragraph{Invariance of emissions.} Centering does not change the
per-position emission scores: $\mathcal{S}_{t,c} -
\mathcal{S}_{t-1,c} = f_\theta(t{-}1, c) - \nu_{b,c}$, and the
constant $\nu_{b,c}$ cancels in the forward recurrence (it shifts all
labels equally at each timestep). However, centering keeps the absolute
magnitude of $\mathcal{S}$ bounded, which matters for content scores
spanning many positions when $K > 1$.

\section{Sequence Boundary Handling}
\label{sec:boundaries}

\texttt{flash-semicrf} supports two boundary mechanisms that control how sequence start and end positions are scored. These interact differently with the $K=1$ and $K=2$ fast paths.

\subsection{Scalar Boundaries (Zero-Overhead Folding)}
\label{sec:scalar-boundaries}

When \texttt{use\_sequence\_boundaries=True}, learnable vectors $\pi^{\text{start}} \in \mathbb{R}^C$ and $\pi^{\text{end}} \in \mathbb{R}^C$ are activated (initialized to zero, equivalent to \texttt{pytorch-crf}'s \texttt{start\_transitions}/\texttt{end\_transitions}). These add label-dependent scores at sequence boundaries: $\pi_c^{\text{start}}$ rewards (or penalizes) label $c$ at the start of a sequence, and $\pi_c^{\text{end}}$ at the end.

Rather than threading boundary parameters through the forward and backward kernels, scalar boundaries are \textbf{folded directly into the cumulative scores} via prefix-sum arithmetic:
\begin{align}
\mathcal{S}_{b,0,c} &\leftarrow \mathcal{S}_{b,0,c} - \pi_c^{\text{start}} \label{eq:fold-start} \\
\mathcal{S}_{b,L_b,c} &\leftarrow \mathcal{S}_{b,L_b,c} + \pi_c^{\text{end}} \label{eq:fold-end}
\end{align}

\paragraph{Why this works.} The content score for a segment $[s, t)$ is $\mathcal{S}_{t,c} - \mathcal{S}_{s,c}$. By subtracting $\pi^{\text{start}}$ from position 0:
\begin{itemize}
    \item Any segment starting at $s = 0$ gets content score $(\mathcal{S}_{t,c} - \mathcal{S}_{0,c}) + \pi_c^{\text{start}}$.
    \item Segments not starting at 0 are unaffected (their $\mathcal{S}_{s,c}$ and $\mathcal{S}_{t,c}$ both remain unchanged).
\end{itemize}
Similarly, adding $\pi^{\text{end}}$ to position $L_b$ adds it to any segment ending at $t = L_b$.

\paragraph{Consequence for dispatch.} Because the folding modifies $\mathcal{S}$ \emph{before} the forward/backward functions are called, the $K=1$ and $K=2$ fast paths remain active. The \texttt{LinearCRFStreaming} and \texttt{SemiCRFK2Streaming} autograd functions are boundary-agnostic--they operate on the already-modified $\mathcal{S}$ with zero additional overhead.

\paragraph{Gradient flow.} Gradients with respect to $\pi^{\text{start}}$ and $\pi^{\text{end}}$ flow automatically through the modified cumulative scores via standard autograd:
\begin{align}
\nabla_{\pi_c^{\text{start}}} \mathcal{L} &= -\nabla_{\mathcal{S}_{b,0,c}} \mathcal{L} \\
\nabla_{\pi_c^{\text{end}}} \mathcal{L} &= \nabla_{\mathcal{S}_{b,L_b,c}} \mathcal{L}
\end{align}

\subsection{Full Boundary Projections (Fallthrough to Generic Path)}
\label{sec:full-boundaries}

When \texttt{use\_boundary\_projections=True}, position-dependent boundary scores $\mathcal{P}^{\text{start}}_{t,c}$ and $\mathcal{P}^{\text{end}}_{t,c}$ are computed by linear layers applied to the encoder hidden states. These provide richer expressiveness than scalar boundaries, allowing the model to learn position-specific start/end preferences.

When full boundary projections are active, the $K=1$ and $K=2$ fast paths are \textbf{skipped}. Dispatch falls through to the generic $K \geq 3$ streaming path, which handles boundaries via the \texttt{HAS\_BOUNDARIES} flag in the Triton kernel (or the PyTorch reference equivalent).

\paragraph{Composition with scalar boundaries.} If both \texttt{use\_sequence\_boundaries} and \texttt{use\_boundary\_projections} are enabled, the scalar vectors are added to the projection tensors at the boundary positions:
\begin{align}
\mathcal{P}^{\text{start}}_{b,0,c} &\leftarrow \mathcal{P}^{\text{start}}_{b,0,c} + \pi_c^{\text{start}} \\
\mathcal{P}^{\text{end}}_{b,L_b-1,c} &\leftarrow \mathcal{P}^{\text{end}}_{b,L_b-1,c} + \pi_c^{\text{end}}
\end{align}
In this case the boundary projections trigger the generic path regardless, so the scalars compose additively within that path rather than folding into $\mathcal{S}$.

\paragraph{Rationale for fallthrough.} The full boundary projections require per-position boundary scores to be threaded through the forward recurrence. The generic kernel already supports this via \texttt{HAS\_BOUNDARIES}. Duplicating that logic in the specialized $K=1$/$K=2$ PyTorch paths would add complexity for a niche use case.

\section{K=1: Linear CRF}
\label{sec:linear-crf}

When $K=1$, every segment has duration 1, reducing the Semi-CRF to a standard linear-chain CRF. The forward recurrence simplifies to:

\begin{equation}
\tilde{\alpha}_t(c) = \text{emit}_t(c) + \log \sum_{c'=1}^{C} \exp\bigl(\tilde{\alpha}_{t-1}(c') + \mathcal{T}_{c',c}\bigr)
\end{equation}

where the emission score is computed via prefix-sum difference:
\begin{equation}
\text{emit}_t(c) = \mathcal{S}_{t,c} - \mathcal{S}_{t-1,c} + \mathcal{B}_{1,c}
\end{equation}

When scalar boundaries are active, the boundary scores are already embedded in $\mathcal{S}$ (Section~\ref{sec:scalar-boundaries}), so the recurrence above implicitly includes $\pi_c^{\text{start}}$ at $t=1$ and $\pi_c^{\text{end}}$ at $t=L_b$ without any modification to the algorithm.

\subsection{Initialization Convention}

Following the \texttt{pytorch-crf} convention, we use \textbf{uniform initialization}:
\begin{equation}
\tilde{\alpha}_0(c) = 0 \quad \forall c \in \{1, \ldots, C\}
\end{equation}

This is equivalent to assuming a uniform distribution over initial states. At $t=1$:
\begin{equation}
\tilde{\alpha}_1(c) = \text{emit}_1(c) + \log \sum_{c'=1}^{C} \exp(\mathcal{T}_{c',c})
\end{equation}

\paragraph{Interaction with Sequence Boundaries.} When \texttt{use\_sequence\_boundaries=True}, learnable parameters $\pi^{\text{start}}_c$ and $\pi^{\text{end}}_c$ are activated (initialized to zero). These are \emph{not} applied by modifying $\tilde{\alpha}_0$--instead they are folded into $\mathcal{S}$ via the prefix-sum trick (Section~\ref{sec:scalar-boundaries}). The forward algorithm still sees $\tilde{\alpha}_0(c) = 0$; the boundary effect appears in the emission scores. At initialization ($\pi^{\text{start}} = \mathbf{0}$), behavior is identical to the boundary-free case. After training, $\pi^{\text{start}}$ and $\pi^{\text{end}}$ provide independent expressiveness beyond the transition matrix.

\subsection{Algorithm}

Algorithm~\ref{alg:linear-crf-forward} presents the optimized $K=1$ forward pass.

\begin{algorithm}[H]
\caption{Linear CRF Forward ($K=1$)}\label{alg:linear-crf-forward}
\KwIn{Cumulative scores $\mathcal{S} \in \mathbb{R}^{B \times (T+1) \times C}$,
      Transitions $\mathcal{T} \in \mathbb{R}^{C \times C}$,
      Duration bias $\mathcal{B} \in \mathbb{R}^{1 \times C}$ (optional),
      Lengths $\ell \in \mathbb{Z}^B$}
\KwOut{Log partition $\log Z \in \mathbb{R}^B$}
\BlankLine
$\tilde{\boldsymbol{\alpha}} \gets \mathbf{0} \in \mathbb{R}^{B \times C}$ \tcp*{Uniform initialization}
\BlankLine
\For{$t \gets 1$ \KwTo $T$}{
    $\mathbf{e}_t \gets \mathcal{S}[:, t, :] - \mathcal{S}[:, t-1, :] + \mathcal{B}_{1}$ \tcp*{Emission}
    \BlankLine
    \tcp{Standard linear CRF recurrence}
    $\tilde{\boldsymbol{\alpha}}_{\text{new}} \gets \text{LogSumExp}_{c'}(\tilde{\boldsymbol{\alpha}}[:, c'] + \mathcal{T}_{c', :}) + \mathbf{e}_t$\;
    \BlankLine
    \tcp{Update only active sequences}
    $\tilde{\boldsymbol{\alpha}} \gets \texttt{where}(t \leq \ell, \tilde{\boldsymbol{\alpha}}_{\text{new}}, \tilde{\boldsymbol{\alpha}})$\;
}
\BlankLine
$\log Z \gets \text{LogSumExp}_c(\tilde{\boldsymbol{\alpha}}[\ell])$ \tcp*{Final reduction}
\Return{$\log Z$}\;
\end{algorithm}

\paragraph{Note.} When scalar boundaries are active, $\mathcal{S}$ has already been modified per Equations~\eqref{eq:fold-start}--\eqref{eq:fold-end}. The algorithm is unchanged; the boundary scores appear implicitly in the emission differences.

\subsection{Complexity}

\begin{itemize}
    \item \textbf{Time}: $O(TC^2)$ -- matrix-vector products at each timestep
    \item \textbf{Space}: $O(BC)$ -- single $\alpha$ vector per batch element
    \item \textbf{No checkpointing}: Full $\alpha$ history stored for backward pass ($O(BTC)$)
\end{itemize}

The space overhead for storing full $\alpha$ history is acceptable because $K=1$ implies short-to-moderate sequences where the $O(BTC)$ cost is manageable.

\section{K=2: Near-Linear CRF}
\label{sec:k2-crf}

When $K=2$, segments can have duration 1 or 2. Rather than using the general ring buffer (which has edge cases at $K=2$ due to modular arithmetic with small indices), we use explicit 2-step history variables.

\subsection{Forward Recurrence}

At each position $t$, we combine contributions from both durations:

\begin{equation}
\tilde{\alpha}_t(c) = \text{LogSumExp}\Bigl(
    \underbrace{\text{score}_{k=1}(t, c)}_{\text{duration 1}},\;
    \underbrace{\text{score}_{k=2}(t, c)}_{\text{duration 2}}
\Bigr)
\end{equation}

where:
\begin{align}
\text{score}_{k=1}(t, c) &= \text{emit}_{t-1:t}(c) + \log \sum_{c'} \exp(\tilde{\alpha}_{t-1}(c') + \mathcal{T}_{c',c}) \\
\text{score}_{k=2}(t, c) &= \text{emit}_{t-2:t}(c) + \log \sum_{c'} \exp(\tilde{\alpha}_{t-2}(c') + \mathcal{T}_{c',c})
\end{align}

The emission scores are:
\begin{align}
\text{emit}_{t-1:t}(c) &= \mathcal{S}_{t,c} - \mathcal{S}_{t-1,c} + \mathcal{B}_{1,c} \\
\text{emit}_{t-2:t}(c) &= \mathcal{S}_{t,c} - \mathcal{S}_{t-2,c} + \mathcal{B}_{2,c}
\end{align}

\subsection{Algorithm}

Algorithm~\ref{alg:k2-forward} presents the optimized $K=2$ forward pass using explicit history variables.

\begin{algorithm}[H]
\caption{Semi-CRF Forward ($K=2$)}\label{alg:k2-forward}
\KwIn{Cumulative scores $\mathcal{S}$, Transitions $\mathcal{T}$, Duration bias $\mathcal{B} \in \mathbb{R}^{2 \times C}$, Lengths $\ell$}
\KwOut{Log partition $\log Z \in \mathbb{R}^B$}
\BlankLine
$\tilde{\boldsymbol{\alpha}}^{(1)} \gets \mathbf{0} \in \mathbb{R}^{B \times C}$ \tcp*{$\alpha[t-1]$}
$\tilde{\boldsymbol{\alpha}}^{(2)} \gets -\infty \in \mathbb{R}^{B \times C}$ \tcp*{$\alpha[t-2]$ (invalid at $t=1$)}
\BlankLine
\For{$t \gets 1$ \KwTo $T$}{
    \tcp{Duration $k=1$: segment from $t-1$ to $t$}
    $\mathbf{e}_{k=1} \gets \mathcal{S}[:, t, :] - \mathcal{S}[:, t-1, :] + \mathcal{B}_{1}$\;
    $\mathbf{s}_{k=1} \gets \text{LogSumExp}_{c'}(\tilde{\boldsymbol{\alpha}}^{(1)}[:, c'] + \mathcal{T}_{c', :}) + \mathbf{e}_{k=1}$\;
    \BlankLine
    \If{$t \geq 2$}{
        \tcp{Duration $k=2$: segment from $t-2$ to $t$}
        $\mathbf{e}_{k=2} \gets \mathcal{S}[:, t, :] - \mathcal{S}[:, t-2, :] + \mathcal{B}_{2}$\;
        $\mathbf{s}_{k=2} \gets \text{LogSumExp}_{c'}(\tilde{\boldsymbol{\alpha}}^{(2)}[:, c'] + \mathcal{T}_{c', :}) + \mathbf{e}_{k=2}$\;
        $\tilde{\boldsymbol{\alpha}}_{\text{new}} \gets \text{LogSumExp}(\mathbf{s}_{k=1}, \mathbf{s}_{k=2})$\;
    }
    \Else{
        $\tilde{\boldsymbol{\alpha}}_{\text{new}} \gets \mathbf{s}_{k=1}$\;
    }
    \BlankLine
    \tcp{Shift history}
    $\tilde{\boldsymbol{\alpha}}^{(2)} \gets \tilde{\boldsymbol{\alpha}}^{(1)}$\;
    $\tilde{\boldsymbol{\alpha}}^{(1)} \gets \texttt{where}(t \leq \ell, \tilde{\boldsymbol{\alpha}}_{\text{new}}, \tilde{\boldsymbol{\alpha}}^{(1)})$\;
}
\BlankLine
$\log Z \gets \text{LogSumExp}_c(\tilde{\boldsymbol{\alpha}}^{(1)}[\ell])$\;
\Return{$\log Z$}\;
\end{algorithm}

\subsection{Complexity}

\begin{itemize}
    \item \textbf{Time}: $O(TC^2)$ -- two matrix-vector products per timestep
    \item \textbf{Space}: $O(BC)$ -- two $\alpha$ vectors per batch element
    \item \textbf{No ring buffer}: Explicit variables avoid modular arithmetic overhead
\end{itemize}

\section{Backward Pass and Gradients}

Both $K=1$ and $K=2$ implementations use the standard forward-backward algorithm for gradient computation. Because these paths do not use checkpointing, the full $\alpha$ history is stored during the forward pass.

\paragraph{Boundary transparency.} When scalar boundaries are active, the backward pass requires no special handling because $\pi^{\text{start}}$ and $\pi^{\text{end}}$ have already been folded into $\mathcal{S}$ (Section~\ref{sec:scalar-boundaries}). Gradients with respect to $\pi^{\text{start}}_c$ and $\pi^{\text{end}}_c$ flow automatically through the modified cumulative scores via standard autograd.

\subsection{Marginal Computation}

The marginal probability for a segment spanning $[t-k, t-1]$ with transition $c' \to c$ is:
\begin{equation}
\mu(t, k, c, c') = \frac{\exp\bigl(\tilde{\alpha}_{t-k}(c') + \tilde{\psi}(t, k, c, c') + \tilde{\beta}_t(c)\bigr)}{\exp(\log Z)}
\end{equation}

where the edge potential is:
\begin{equation}
\tilde{\psi}(t, k, c, c') = (\mathcal{S}_{t,c} - \mathcal{S}_{t-k,c}) + \mathcal{B}_{k,c} + \mathcal{T}_{c',c}
\end{equation}

\paragraph{Numerical safeguard.} In the PyTorch reference implementation, the log marginal $\tilde{\alpha}_{t-k}(c') + \tilde{\psi}(t,k,c,c') + \tilde{\beta}_t(c) - \log Z$ is clamped to $[-80, 80]$ before exponentiation. This prevents overflow/underflow when individual terms are large in absolute value but nearly cancel. The clamped range corresponds to $\exp(\pm 80) \approx 5.5 \times 10^{34}$, well within float64 range.

\subsection{Gradient Formulas}

\paragraph{Cumulative Scores.} The gradient for $\mathcal{S}_{t,c}$ accumulates contributions from segments ending at $t$ (positive) and starting after $t$ (negative):
\begin{equation}
\nabla \mathcal{S}_{t,c} = \sum_{k, c'} \mu(t, k, c, c') - \sum_{k, c'} \mu(t+k, k, c, c')
\end{equation}

\paragraph{Transitions.} The gradient sums marginals over all positions and durations:
\begin{equation}
\nabla \mathcal{T}_{c',c} = \sum_b \frac{\partial \mathcal{L}}{\partial Z_b} \cdot \sum_{t,k} \mu_b(t, k, c, c')
\end{equation}

\paragraph{Duration Bias.} The gradient for duration $k$ sums over all segments of that duration:
\begin{equation}
\nabla \mathcal{B}_{k,c} = \sum_b \frac{\partial \mathcal{L}}{\partial Z_b} \cdot \sum_{t, c'} \mu_b(t, k, c, c')
\end{equation}

\section{Viterbi Decoding}

Both $K=1$ and $K=2$ implementations support Viterbi decoding for MAP inference. The algorithms follow the same structure as the forward pass, replacing LogSumExp with max and maintaining backpointers.

\subsection{K=1 Viterbi}

Standard linear CRF Viterbi with backpointers for the best previous label at each position.

\subsection{K=2 Viterbi}

Extended backpointers track both the best previous label \textbf{and} the best duration (1 or 2) at each position. The traceback reconstructs the segmentation by stepping back by the recorded duration.

\section{Why Triton Kernels Require $K \geq 3$}\label{sec:triton_k3}

The Triton streaming kernel is designed for the general Semi-CRF case where ring buffers and checkpointing provide substantial memory savings. For small $K$, ring buffer indexing creates edge cases:

\begin{itemize}
    \item \textbf{$K=1$}: All positions map to ring index 0. The ring buffer degenerates to a single slot, providing no benefit over a simple variable.

    \item \textbf{$K=2$}: Indices alternate between 0 and 1. Wraparound logic is fragile and provides minimal memory savings over explicit variables.

    \item \textbf{$K \geq 3$}: Ring buffer provides meaningful history separation. Checkpointing amortizes over multiple segments, justifying the implementation complexity.
\end{itemize}

The PyTorch implementations for $K=1$ and $K=2$ are already efficient (no kernel launch overhead, simple control flow), making Triton acceleration unnecessary.

\section{Implementation Summary}

\begin{table}[h]
\centering
\begin{tabular}{@{}lccccc@{}}
\toprule
$K$ & Time & Space & Checkpointing & Ring Buffer & Backend \\
\midrule
1 & $O(TC^2)$ & $O(BC)$ & No & No & PyTorch \\
2 & $O(TC^2)$ & $O(BC)$ & No & No & PyTorch \\
$\geq 3$ & $O(TKC^2)$ & $O(BKC)$ & Yes & Yes & Triton (GPU) / PyTorch (CPU) \\
\bottomrule
\end{tabular}
\caption{Implementation characteristics by maximum segment duration.}
\end{table}

\paragraph{Automatic Dispatch.} The \texttt{semi\_crf\_streaming\_forward} function automatically routes to the appropriate implementation based on $K$. Users need not manage dispatch manually.

\section{Functional Equivalence}

All three implementations ($K=1$, $K=2$, $K \geq 3$) compute the same partition function and gradients for their respective segment duration constraints. Specifically:

\begin{itemize}
    \item The $K=1$ path produces identical results to the general streaming kernel with $K=1$ (verified numerically).
    \item The $K=2$ path produces identical results to the general streaming kernel with $K=2$ (verified numerically).
    \item All paths support variable-length batches with proper masking.
    \item The equivalence between $K=1$/$K=2$ fast paths and the generic kernel holds \textbf{only when full boundary projections are not active}. When $\mathcal{P}^{\text{start}}$ or $\mathcal{P}^{\text{end}}$ are provided, dispatch bypasses the specialized paths entirely (Section~\ref{sec:full-boundaries}), so the question of equivalence does not arise. Scalar boundaries ($\pi^{\text{start}}$, $\pi^{\text{end}}$) are folded into $\mathcal{S}$ before dispatch and do not affect equivalence.
\end{itemize}

The specialized paths exist purely for performance optimization, not behavioral differences.


\let\section\origsection
\let\subsection\origsubsection
\let\subsubsection\origsubsubsection

\end{appendices}

\bibliography{flash_semicrf}

\end{document}